\documentclass[11pt]{article}

\usepackage{url}
\usepackage{fullpage}
\usepackage{color}
\usepackage{amsmath,amssymb,amsthm}
\usepackage{amsmath,amssymb,mathtools}
\usepackage{bm}
\usepackage{amsmath}
\usepackage{amssymb}
\usepackage{latexsym}
\usepackage{verbatim}
\usepackage{subfigure}
\usepackage[final]{graphicx}
\usepackage{psfrag}
\usepackage{xcolor}
\usepackage{lipsum}
\usepackage{authblk}
\usepackage{blindtext}
\usepackage{geometry}
\usepackage{amsthm}
\usepackage{lipsum}

\newcommand\blfootnote[1]{%
  \begingroup
  \renewcommand\thefootnote{}\footnote{#1}%
  \addtocounter{footnote}{-1}%
  \endgroup
}
\renewcommand{\cal}{\mathcal}

\newcommand\cB{{\mathcal B}}
\newcommand{\cC}{{\cal C}}
\newcommand{\cD}{{\cal D}}

\newcommand{\cF}{{\cal F}}

\newcommand{\cS}{{\mathcal S}}

\newcommand{\bE}{\mathbb{E}}

\newcommand{\bN}{\mathbb{N}}
\newcommand{\bP}{\mathbb{P}}

\newcommand{\bR}{{\mathbb R}}
\newcommand{\R}{{\mathbb R}}

\renewcommand{\leq}{\leqslant}
\renewcommand{\geq}{\geqslant}

\newcommand{\cyan}{\color{cyan}}

\usepackage{algorithm}

\usepackage{algpseudocode}
\usepackage{cite}

\DeclareMathOperator*{\E}{\mathbb{E}}
\let\Pr\relax
\DeclareMathOperator*{\Pr}{\mathbb{P}}

\DeclareMathOperator{\argmax}{argmax}

\DeclareMathOperator{\Var}{Var}

\def\pp{\mathbb{P}}
\def\Pro{\mathbb{P}}

\def\rr{\mathbb{R}}
\def\ee{\mathbb{E}}

\usepackage{scalerel,stackengine}
\stackMath
\newcommand\reallywidehat[1]{%
\savestack{\tmpbox}{\stretchto{%
  \scaleto{%
    \scalerel*[\widthof{\ensuremath{#1}}]{\kern-.6pt\bigwedge\kern-.6pt}%
    {\rule[-\textheight/2]{1ex}{\textheight}}
  }{\textheight}%
}{0.5ex}}%
\stackon[1pt]{#1}{\tmpbox}%
}
\newcommand{\algone}{Scaffolding Set Algorithm}
\newcommand{\algfull}{Meta-Algorithm }
\newcommand{\scaffolding}{scaffolding}
\newcommand{\Scaffolding}{Scaffolding}
\newcommand{\remove}[1]{}
\newcommand{\later}[1]{} 

\newcommand{\C}{{\cal C}}
\newcommand{\D}{{\cal D}}
\newcommand{\SC}{{\cal S}}
\newcommand{\G}{{\cal G}}

\newcommand{\red}{\color{red}}

\newcommand{\beq}{\begin{equation}}
\newcommand{\eeq}{\end{equation}}
\newcommand{\beas}{\begin{eqnarray*}}
\newcommand{\eeas}{\end{eqnarray*}}
\newcommand{\bea}{\begin{eqnarray}}
\newcommand{\eea}{\end{eqnarray}}
\newcommand{\bei}{\begin{itemize}}
\newcommand{\eei}{\end{itemize}}
\newcommand{\ben}{\begin{enumerate}}
\newcommand{\een}{\end{enumerate}}

\newtheorem{Theorem}{Theorem}

\newtheorem{Corollary}[Theorem]{Corollary}
\newtheorem{Proposition}[Theorem]{Proposition}
\newtheorem{Lemma}[Theorem]{Lemma}
\newtheorem{Definition}{Definition}

\newtheorem{Remark}[Theorem]{Remark}

\newtheorem{Assumption}[Theorem]{Assumption}

\newcommand{\nn}{\mathbb N}
\newcommand{\Prob}{{\mathbb{P}}}

\usepackage{parskip}
\usepackage[margin=1.1cm]{caption}

\begin{document}
\thispagestyle{empty}

\title{\Scaffolding{} Sets}

\author{Maya Burhanpurkar \footnote{Author names are listed in alphabetical order.} \thanks{Harvard University. E-mail: burhanpurkar@college.harvard.edu}  
   \quad  Zhun Deng\thanks{Harvard University. E-mail: zhundeng@g.harvard.edu}
	\quad Cynthia Dwork\thanks{Harvard University.
	E-mail: dwork@seas.harvard.edu}
	\quad Linjun Zhang\thanks{Rutgers University. 
	E-mail: linjun.zhang@rutgers.edu}
}

\date{}
\maketitle
\blfootnote{This work was supported, in part, by NSF DMS-2015378, NSF CCF-1763665 and grants from the Simons and Sloan Foundations.}


\thispagestyle{empty}

\abstract

{\em Predictors} map individual instances in a population to the interval $[0,1]$.  
For a collection $\C$ of subsets of a population, a predictor is multi-calibrated with respect to $\C$ if it is simultaneously calibrated on each set in~$\C$.  
We initiate the study of the construction of {\em \scaffolding{} sets}, a small collection $\SC$ of sets with the property that multi-calibration with respect to $\SC$ ensures {\em correctness}, and not just calibration, of the predictor.  Our approach is inspired by the folk wisdom that the intermediate layers of a neural net learn a highly structured and useful data representation.


\clearpage
\pagenumbering{arabic} 
\section{Introduction}

\label{sec:intro}
Prediction algorithms ``score" individual instances, mapping them to numbers in $[0,1]$ that are popularly understood as ``probabilities'' or ``likelihoods'': the probability of success in the job, the likelihood that the loan will be repaid on schedule, the chance of recidivism within 2 years.  Setting aside for a moment the pressing question of the {\em meaning} of a probability for a non-repeatable event, these scores have life-altering consequences, and there is great concern that they be {\em fair}.  One approach to defining fairness is through {\em calibration}: a predictor $p$ is calibrated on a set $S$ for a distribution $\D$ on (instance, Boolean label) pairs $(x,y)$, if $\E_{(x,y) \sim \D} [y~|~p(x)=v, x\in S] = v$.  Calibration as an accuracy condition (in the online setting) has a long intellectual history~\cite{Dawid1982,FosterVohra,G3}; calibration as a fairness criterion was introduced by Kleinberg, Mullinathan, and Raghavan~\cite{KPR}, who required calibration simultaneously on disjoint demographic groups.  The intuition expressed in that work is compelling: a score of $v$ ``means the same thing'' independent of the demographic group of the individual being scored. 

  It is well known that calibration is a poor measure of accuracy; for example, as Fienberg and DeGroot (and, as they note, others before them) point out~\cite{degroot1983comparison}, in the online setting calibration can be a {\em disincentive} for accurate forecasting, and Sandroni showed more generally that, if a quality test is known in advance, a charlatan who is unable to predict can nevertheless pass the quality test (calibration is a special case of a quality test)~\cite{Sandroni}.  Nonetheless, in a seminal work, H\'ebert-Johnson, Kim, Reingold, and Rothblum take an exceptional step in the direction of accuracy through {\em multi-calibration} with respect to a (possibly large) collection $\mathcal S$ of arbitrarily intersecting large sets~\cite{HKRR}. Here the requirement is that the predictor be (approximately) calibrated simultaneously on all sets $S \in \mathcal S$. The authors demonstrate a batch method for creating a ``small'' multi-calibrated predictor from relatively few training observations.  At a very high level, the algorithm begins with an arbitrary candidate predictor $p$ and makes a series of ``adjustments'' when a pair $(S,v)$ is found, for $S \in \mathcal S$ and $v \in [0,1]$, such that $p$ is not calibrated (as estimated on the training data) on $S_v =\{x | x \in S,~p(x)=v\}$ and $S_v$ is of sufficiently large weight according to the underlying distribution $\D$.  By means of a potential function argument, the authors show that every adjustment brings the modified predictor substantially closer (in $\ell_2$ norm) to the ``truth,'' with the size of advance tied to the quality of the approximation to calibration.  Since the value of the potential function is bounded {\it a priori}, it follows that not too many adjustments are required before ``truth'' is well-approximated.  Multi-calibrated algorithms have performed well in medical settings~\cite{barda2020developing, barda2021addressing}.

  Nevertheless, reaching ``truth'' is not the expected outcome; for example, the representation of individuals to the algorithm might be missing key features necessary for computing the probability of a positive outcome, or the collection of sets~$\mathcal S$ only contains sets that are orthogonal to the probability of positive outcome, in which case the predictor that reports the population average is multi-calibrated with respect to~$\mathcal S$.   \cite{dwork2021outcome} views the design goal of a predictor through the lens of indistinguishability: multi-calibrated predictors provide a generative model of the real world whose outcomes are ``indistinguishable,'' to a collection of distinguishers intimately tied to the sets in~$\mathcal S$, from those seen in the real world.  This is not the same as producing ``true'' probabilities of positive outcomes, even if we had a solution to the philosophical problem of how to define such probabilities for non-repeatable events.  

The power, and the computational complexity, of multi-calibration comes from choosing a rich collection $\mathcal S$. Suppose, for example, that $\mathcal S$ contains all sets recognizable by circuits of size, say, $n^2$. Consider a demographic group that is so compromised by a history of subordination that individuals in this group accept their mistreatment as ``justified''~\cite{Banaji}.  If membership in this group can be recognized by a circuit of size $n^2$ then multi-calibration with respect to $\mathcal S$ immediately ensures calibration on this demographic group, without requiring members of the group to (know to) advocate for the group.  At the same time, however, the set of all circuits of size $n^2$ is very large, and the step, in the algorithm of~\cite{HKRR}, of finding a (slice of a) set to trigger an adjustment could potentially require exhaustive search of $\mathcal S$; (\cite{HKRR} shows that learning a predictor that is multi-calibrated with respect to $\mathcal S$ is as hard as weak agnostic learning for $\mathcal S$; see~\cite{gerrymandering} for a related result).


These dueling considerations -- power and complexity -- together with the fact that a perfectly accurate predictor is calibrated on all sets, motivate us to consider the following question: {\em Can we find a small collection of sets $\mathcal S$, such that multi-calibrating on $\mathcal S$ guarantees accuracy?} We call this the {\em \Scaffolding{} Set Construction} problem, and our principal contribution is {a proof of concept: a \scaffolding{} set construction algorithm
that works under a variety of assumptions common in the machine learning literature.}

Suppose there exists a ground truth $p^*$ mapping instances $x$ (say, students) to true probabilities $p^*(x)$ (say, of graduating within 4 years).  Fix an arbitrary grid on the interval $[0,1]$, specifically, multiples of some small $\gamma$, {\it e.g.}, $\gamma = 0.01$, and assume for simplicity that for all $x$, $p^*(x)$ is a multiple of $\gamma$.  Our starting point is the observation that {\em if the collection $\mathcal S$ contains the level sets of $p^*$ (possibly among other sets), then multi-calibration with respect to $\mathcal S$ ensures closeness to~$p^*$}. 

\paragraph{Our Contributions.}
Speaking informally, our goal is to find a small collection of sets with the property that multi-calibration with respect to this small collection ensures accuracy (closeness to $p^*$).  By virtue of this smallness, multi-calibration with respect to this collection can be carried out in polynomial time using the algorithm in~\cite{HKRR} or any other multi-calibration algorithm that is efficient in the number of sets, giving rise to a natural 2-step ``Meta-Algorithm" (Algorithm~\ref{alg:full}). 
Throughout this work We assume individuals are presented to the algorithm(s) as elements of $\rr^d$ for some $d>0$\footnote{There is a difference between the instance (say, student) and the representation of the instance to the algorithm (say, transcript, test scores, and list of extra-curricular activities).  We have in mind the situation in which the representation is rich, so distinct instances result in distinct representations to the algorithm, but our results hold in the general case.}.
\begin{enumerate}
  \item
We formalize the problem of learning {\em \scaffolding{} sets} as, roughly speaking, the generation of a polynomial-sized collection $\SC$ of sets that contains the (approximate) level sets of $p^*$ (the truth is more nuanced, but this description suffices for building intuition).

A {\em \scaffolding{} set} algorithm, when run on a data representation mapping $\hat h: \rr^d \rightarrow  \rr^r$, $r < d$ (more on this below)
and $d$-dimensional training data labeled with Boolean outcomes, is said to {\em succeed} if it outputs a small collection of sets $\SC$ that indeed contains the (approximate) level sets of $p^*$.

\item
  We provide an algorithm for the \Scaffolding{} Set problem and give sufficient conditions on the representation mapping for the algorithm's success.  
  The \Scaffolding{} Set Construction algorithm (Algorithm~\ref{alg:set}) takes a (typically low-dimensional) representation mapping as input, without regard to how this mapping was obtained. The algorithm does not require that $p^*$ be expressible by the class of architectures that can express the representation (Section~\ref{set construction}).
 
  Our approach is motivated by the folk wisdom that, in a classification task, the initial layers of a neural network extract a useful data representation for the task at hand, while the final layers of the network are responsible for assigning ``probabilities" to those representations. 
  
{Under the above-mentioned sufficient conditions, we prove asymptotic convergence of our method: as the number of samples approaches infinity the calibration error goes to~$0$;  moreover, under the same conditions, we present a multi-calibration method that, when used in the multi-calibration step of the Meta-Algorithm, yields  
a nearly minimax rate-optimal predictor (Section~\ref{sec:accuracy}).} 

\item
\label{contribution:find rep}
We provide methods for finding a suitable representation that applies to a wide class of data generation models $\ee[Y|X]=p^*(x)$
(Section~\ref{sec: learning h}), including a $k$-layer neural network model. We further obtain results for the transfer learning setting, where, in addition to the sample from the target task, we also have abundant samples from source tasks. 

\item
If, using our techniques, we can learn a collection $\mathcal S$ of \scaffolding{} sets such that multi-calibration with respect to $\mathcal S$ yields a good approximation to $p^*$, then perhaps we can directly learn $p^*$, and not bother with all this machinery.  In other words, do the techniques developed in this work actually buy us anything?  In short, the answer is yes.
   We give an example of a (non-contrived) data generation model for which, by using a neural net of a given architecture, together with training data, we can solve the \scaffolding{} set problem for $p^*$ even though $p^*$ itself cannot be computed by a neural net of this architecture (Section~\ref{sec: benefit}).  The representation mapping is found using the method of Contribution~\ref{contribution:find rep} above.  

\item  {\bf Experiments.} Using synthetic data drawn from a known distribution on $p^*$, we show experimentally that multi-calibrating with respect to the collection of sets obtained by our \Scaffolding{} Set Construction algorithm trained with Boolean labels drawn according to $p^*$ gets closer to the truth than training a neural network on the same data. 

  \end{enumerate}
  Suppose we have a specific and moderately-sized collection $\mathcal C$ of sets for which we wish to ensure calibration; for example, these may be specific demographic groups.  Recall that we can not know when the \Scaffolding{} Set algorithm produces level sets.  When it does, multi-calibration produces a predictor that is close to $p^*$ (Theorem~\ref{thm:mc wrt S}), which (most likely) gives calibration with respect to $\mathcal C$.  Since we cannot be certain of success, we can still explicitly calibrate with respect to $\mathcal C$ using the algorithm of~\cite{HKRR} as a ``post-processing" step; \cite{HKRR} argues that such post-processing does not harm accuracy.
{We remark that, while it is generally impossible to ensure or detect success of any \scaffolding{} set algorithm, 
given any collection $\G$ of (arbitrarily intersecting) demographic groups we can ensure multi-calibration for $\mathcal G$ by, for example, setting $\C = \G \cup \SC$ and using the algorithm in~\cite{HKRR} to multi-calibrate with respect to~$\C$.  If the \scaffolding{} set construction was successful, then in addition to being multi-calibrated with respect to $\mathcal G$ the predictor will also be accurate.
 }

  \subsection{Additional Related Work}
  The formal study of algorithmic fairness was initiated by Dwork {\it et al.} in~\cite{FtA} in the context of classification. \cite{FtA} emphasized {\em individual}, rather than {\em group} fairness, requiring that individuals who are similar, under some task-specific notion of similarity, receive similar distributions on classifications, and pointed to flaws of group fairness guarantees as solution concepts. \cite{FtA} further provided an algorithm for ensuring fair classification, given a similarity metric; nearly a decade later, Ilvento took a major step toward learning a metric from an expert~\cite{Ilvento}; \cite{jung2019eliciting} elicits this kind of information from a group of stakeholders. In the intervening years most (but not all) theoretical research focused on group fairness criteria.  Two independent works suggested multi-group criteria as an approach to bridging group vs individual fairness~\cite{gerrymandering,HKRR}.  In particular, H\'ebert-Johnson {\it et al.}\,introduced the method for constructing efficient multi-calibrated predictors in the batch setting~\cite{HKRR}.  Multi-calibration in the online setting is implicit in~\cite{G3}, \cite{momentmulticalibration} extends the notion of multi-calibration to higher moments and provides constructions, and~\cite{online} provides an efficient online solution. Multi-calibration has also been applied to solve several flavors of problems: fair ranking~\cite{ranking}; {\em ommiprediction}, which is roughly learning a predictor that, for a given class of loss functions, can be post-processed to minimize loss for any function in the class~\cite{omnipredictors}; and providing an alternative to propensity scoring for the purposes of generalization to future populations~\cite{propensity}.

   Calibration, as a measure of predictive uncertainty quantification, is a well-studied concept in statistics and econometrics, see \cite{dawid1982well, foster1997calibrated, sandroni2003calibration,foygel2021limits} and the reference therein. Recently, it was found that although modern machine learning methods, such as neural networks, made significant progress with respect to the predictive accuracy in a variety of learning tasks \cite{simonyan2014very,srivastava2015highway,he2016deep}, their predictive uncertainties are poorly calibrated \cite{guo2017calibration}. To mitigate this issue, several recalibration methods and their variants have been proposed, including scaling approaches \cite{platt1999probabilistic, zadrozny2002transforming,guo2017calibration}, binning approaches \cite{zadrozny2001obtaining,naeini2014binary,gupta2021top}, scaling-binning \cite{kumar2019verified}, and data augmentation \cite{thulasidasan2019mixup, zhang2021and}.  Recent work has begun to develop the theoretical analysis of these methods.  For example, \cite{gupta2021distribution} draws the connection between calibration and conformal prediction; \cite{gupta2021distributionfree} avoids sample splitting in the binning method; and \cite{zhao2021calibrating} analyzes multi-class calibration from the decision making perspective.
  The folklore that intermediate layers of neural net serve as a (low-dimensional) representation \cite{bengio2013representation} has been recently studied, eg., in the settings of transfer learning \cite{du2020few,tripuraneni2020theory, tripuraneni2021provable,deng2021adversarial} and self-supervised learning \cite{lee2020predicting,ji2021power,tian2021understanding}. Those works focus on analyzing prediction error. In this paper, we initiate the study of the role of representation learning in multi-calibration.

\section{Preliminaries}

\subsection{Notation}
  We say an $\R^d$-valued random
vector $x$ follows a $d$-dimensional elementwise sub-Gaussian distribution if for $\mu = \E[x]$ and  $v\in\R^d$, we have
$$\E \exp (\langle\lambda x-\mu, v\rangle) \le \exp(\lambda^2\sigma^2), \forall \lambda \in \R, v\in\R^d, \|v\|=1.$$

Let us denote $X\in \mathcal X$ as the feature vector (typically $\mathcal X=\R^d$), $Y\in\R$ 
as the response (what we are trying to predict). 
For two positive sequences $\{a_k\}$ and $\{b_k\}$, we write $a_k =O(b_k)$ (or $a_n\lesssim b_n$), and $a_k = o(b_k)$, if $\lim_{k\rightarrow\infty}(a_k/b_k) < \infty$ and $\lim_{k\rightarrow\infty}(a_k/b_k) =0$ respectively. $\tilde O(\cdot)$ denotes the term, neglecting the logarithmic factors. We also write $a_k=\Theta(b_k)$ (or $a_k\asymp b_k$) if $a_k\lesssim b_k$ and $b_k\lesssim a_k$. We use $O_p(\cdot)$ to denote stochastic boundedness: a random variable $Z_k=O_p(a_k)$ for some real sequence $a_k$ if $\forall \epsilon>0$, there exists $M,N>0$ such that if $k>N$, $\bP(|Z_k/a_k|>M)\le \epsilon$. For a rank-$r$ matrix $M$, we use $\lambda_1(M)$ or $\lambda_{\max}(M)$ to denote its maximum singular value, and $\lambda_r(M)$ or $\lambda_{\min}(M)$ to denote its smallest singular value. We use $c, c_1,c_2,...$ and $C, C_1,C_2,...$ to denote generic positive constants that may vary from place to place.

\paragraph{Data Generation Model.}
Throughout this work we assume a data generation model of the form $(X,Y)$, where $X \in \rr^d$, $Y \in \{0,1\}$, and for some function $p^*:\rr^d \rightarrow [0,1]$, we have 
$E[Y|X]= p^*(X). $
Our results are easily extended from this {\em classification setting} to the {\em regression setting}.  
We are agnostic as to whether $p^*$ is integer-valued or can take values in $(0,1)$.

\begin{Remark}
There is a difference between the instance (say, student) and the representation of the instance to the algorithm (say, transcript test scores, and list of extra-curricular activities).  We have in mind the situation in which the representation is rich, so distinct instances result in distinct representations (in $\rr^d$) to the algorithm, although our results hold in the more general case as well.  When there are no collisions in representation the philosophical meaning of an individual probability $p^*(x)$ is a topic of study -- what is the probability of a non-repeatable event?  Do non-integer probabilities exist, or are there only outcomes?  See~\cite{dawid2017individual,dwork2021outcome} and the references therein.  Our results are agnostic on these points.
\end{Remark}

\subsection{Calibration/Multi-calibration}
We first introduce the notions of calibration and  multi-calibration. We assume we have samples $D=\{(X_i,Y_i)\}_{i=1}^{m}$ drawn from the joint distribution $(X,Y)$ where $m$ is the sample size.

\begin{Definition}[Calibration]
We say a predictor $\hat p$ is $\alpha$-calibrated   if for all $v\in\R$, $$
|\bE_{X,Y}[Y-\hat p(X)\mid \hat p(X)=v]|\le\alpha.
$$
\end{Definition}

\begin{Definition}[Asymptotic calibration]
Consider a family of predictors $\{{\hat p}_m\}_{m \in \nn}$,
where $\hat p_m$, $m \ge 1$, is constructed based on a sample of size $m$. We abuse notation slightly and say that $\hat p_m$ is {\em asymptotically calibrated} if $\hat p_m$ is $\alpha_m$-calibrated, and $\alpha_m\to 0$ when the sample size $m\to \infty$. We will usually omit the subscript $m$ and write the predictor as $\hat p$.
\end{Definition}

\begin{Definition}[Multi-calibration\cite{HKRR}\footnote{{Our definition is slightly stronger than that in~\cite{HKRR}.}}] 
For a collection of subsets $\mathcal S=\{S_1, ..., S_K\}\subset \mathcal X$. We say a predictor $\hat p$ is $\alpha$-multi-calibrated with respect to $\mathcal S$  if for all $v\in\R$, $$
|\bE_{X,Y}[Y-\hat p(X)\mid \hat p(X)=v, X\in S_k]|\le\alpha, \text{ for all $k\in[K]$}.
$$
\end{Definition}

Similarly, we say a predictor $\hat p$ is \textit{asymptotically multi-calibrated} with respect to $\mathcal S$  if $\alpha\to 0$ when the sample size $m\to \infty$.

A caveat of the calibration/multi-calibration definition is that small calibration error $\alpha$ does not necessarily imply high-accuracy. For example, letting $\hat p(x)=\sum_{i=1}^m Y_i/m$ will yield $O_P(1/\sqrt m)$-calibrated predictor, but in general this $\hat p$ is not accurate as it does not depend on $x$ at all.

Our goal is to find a small collection of subsets  $\mathcal S=\{S_1, ..., S_K\}\subset \mathcal X$ that produces a multi-calibrated predictor with some accuracy guarantee, through the framework of representation learning. We present an algorithm for finding such a collection $\mathcal S$ and describe conditions, summarized in Assumption~\ref{assumption on h} below, under which any predictor that is asymptotically multi-calibrated with respect to $\mathcal S$ has small error.  Under an additional {\em Learnability} assumption (Remark~\ref{rem:learnability}) the predictor is consistent (has estimation error that goes to~0). Further, we propose a predictor that is both multi-calibrated and nearly optimal in terms of accuracy (in the minimax sense) under these conditions. We give concrete and rich examples under which these assumptions hold.  

\subsection{Problem setup}
\begin{Definition}[\Scaffolding{} Set Problem]
For a mapping $p^*: \rr^d \rightarrow [0,1]$ and a data distribution $\mathcal D$ on pairs $(X,Y) \in \rr^d \times \{0,1\}$, where $\ee[Y|X]=p^*(X)$, the {\em \scaffolding{} set problem} is to find, for some fixed polynomial $q$, a method that, given a set of i.i.d.\,draws $\{(X_i,Y_i)\}_{i \in [m]}$ from $\mathcal D$, outputs a collection of sets $\mathcal S^{(m)} = \{S_1, \dots,S_K\}$, $K \le q(m)$, such that $\forall \hat p:\rr^d\rightarrow[0,1]$ that is multi-calibrated with respect to $\mathcal S^{(m)}$, $\ee_X[(\hat p (X) - p^*(X))^2] \rightarrow 0$ as $m \rightarrow \infty$.  We call the sets $\mathcal S^{(m)}$ the {\em \scaffolding{} sets}.  We will usually omit the superscript $(m)$.
\end{Definition}

 To obtain our solution, we break the problem into two parts.  The first takes training data and a {\em representation mapping} $\hat h: \rr^d \rightarrow \rr^r$ as input and outputs a collection of \scaffolding{} sets, where $r$ is the dimension of the representation mapping (we assume $r < d$). The algorithm succeeds under certain technical assumptions on $\hat h$ (see Section~\ref{sec: learning h}); here {\em success} admits error parameters specific to $\hat h$ and the calibration error.
The second part is to find a method for generating $\hat h$ that satisfies these assumptions, with the $\hat h$-specific error parameter vanishing as the size of the training set used to find $\hat h$ grows to infinity.

  In practice, for unknown distributions $p^*$ we cannot verify that the above-mentioned technical assumptions 
  hold. However, we show that for {\em any} predictor $\hat p_0$, {our proposed algorithm (of making $\hat p_0$ multi-calibrated with respect to the \scaffolding{} sets $\mathcal S$)} will {\em never} harm the accuracy of estimating $p^*$, even if neither assumption holds.

\section{Main Results}
\label{sec:main_results}
\subsection{The \Scaffolding{} Set Algorithm and its Properties}

Consider a  predictor $\hat p_0$ of form 
\begin{equation}
    \label{eq:p0 hat}
\hat p_{0}(x)=\hat w\circ \hat h(x),
\end{equation}
where $\hat h: \R^d\to \R^r$ and $\hat w: \R^r\to\R$ are two mappings, where $r$ is the dimension of $\hat{h}$ (we assume $r<d$). Such a form can be obtained, for example, through fitting single/multiple index models\cite{hardle2004nonparametric, eftekhari2021inference} or the learning of neural networks; in the latter case, $\hat h$ is the first to some intermediate layer and $\hat w$ is the remaining layers.  One can interpret $\hat h$ as a (low-dimensional) representation function, in accord with recent papers in representation learning theory in neural nets \cite{du2020few,lee2020predicting, tripuraneni2021provable,ji2021power}. For example, in the numerical investigations by Tripuraneni, Jin, and Jordan \cite{tripuraneni2021provable}, $r=5$.

In this section, we make no assumptions about the accuracy of the predictor $\hat p_0$, nor are we concerned with how it is obtained. In fact, our algorithm will only make use of $\hat h$; we ignore $\hat w$ entirely and everything we say applies to an arbitrary $\hat h: \R^d\to \R^r$. Nonetheless, it is useful to think of $\hat h$ as a learned representation mapping, and we make use of this thinking in Section~\ref{sec: learning h}.

The \Scaffolding{} Set Construction Algorithm takes $\hat h$ and the training data $D$ as input and returns a collection ${\cal S}$ of sets.  We define a second predictor, $\hat p$, based on this set collection, and argue that $\hat p$ is multi-calibrated with respect to~${\cal S}$ (Section~\ref{set construction}).
We then give conditions on $\hat h$ under which $\hat p$ is guaranteed to approximate~$p^*$, giving tight accuracy bounds (Section~\ref{sec:accuracy}). In Section~\ref{sec: learning h} we address the problem of finding a suitable~$\hat h$.

\subsubsection{Algorithm Description}
\label{set construction}

We now describe a Meta-Algorithm (Algorithm 2) for computing our predictors.
Let $\hat h$ be as in Equation~\ref{eq:p0 hat} above, and denote the observations by $D=\{( X_i, Y_i)\}_{i=1}^{m}$ with $X_i\in\R^d$ and 
$Y_i \in \{0,1\}$. The Meta-Algorithm begins by splitting the training samples into two disjoint sets $D=D_1\cup D_2$, that is,  $D_1=\{(X_i^{(1)},Y_i^{(1)})\}_{i=1}^{m_1}$ and  $D_2=\{(X_i^{(2)},Y_i^{(2)})\}_{i=1}^{m-m_1}$ for $m_1\le m$.  
We will use $D_1$ together with $\hat h$ to construct the \scaffolding{} sets and $D_2$ for calibration. Such a sample splitting step is commonly used in machine learning and statistics, such as in the problems of the conformal prediction \cite{lei2018distribution, romano2019conformalized}, standard calibration \cite{kumar2019verified,zadrozny2001obtaining}, and random forest \cite{arlot2014analysis,mourtada2020minimax}.

 In addition to the (possibly learned) representation $\hat h$ and the training data $D$, the \Scaffolding{} Set construction algorithm takes as input a number $B$, which denotes the number of partitions on each coordinate.
 For simplicity, let us assume that, for some $C>0$, ${{\hat h}}(X_i)\in[-C,C]^r$ for all $i \in [n]$. We also let $h_j(X_i)$ denote the $j$-th coordinate of $\hat h(X_i)$ for $j\in[r]$. 
 \begin{figure}[ht]
  \centering
  \begin{minipage}{.9\linewidth}
\begin{algorithm}[H]
\caption{{\algone}}
    \label{alg:set}
   
      \textbf{Input:} 
      The representation $\hat h$, the number of partitions on each coordinate $B$, any dataset $\{(X_i,Y_i)\}_{i=1}^{n}$ with sample size $n\in\bN^+$
      
   \textbf{Step 1:} Let $Output=\{[-C,C]^r\}$ 
   
\textbf{Step 2:} For $j=1:r$
  
\quad\quad\quad\quad\quad For $set^{o}\in Output$

\quad\quad\quad\quad\quad\quad Divide $set^{o}$ into $B$ sets $set^{o}=set^{o}_1\cup set^{o}_2\cup...\cup set^{o}_B$, where the $b$-th set $set^{o}_b$ corresponds to the $(b-1)/B$ to the $b/B$-th quantile of $\{\hat  {h}_j(X_i)\}_{i=1}^{n}$

\quad\quad\quad\quad\quad\quad Replace the $set^{o}$ with $set^{o}_1, set^{o}_2,..., set^{o}_B$ in $Output$

\quad\quad\quad\quad\quad EndFor

\quad\quad\quad\quad EndFor
 
\textbf{Step 3:} For $Output=\{S^{({{\hat h}})}_1,...,S^{({{\hat h}})}_K\}$, define the corresponding partitions in $\mathcal X$: $S_k=\{x\in \mathcal X: \hat h(x)\in S^{(h)}_k\}$, $k \in [K]$, and let $\mathcal S=\{S_1,...,S_K\}$. 

\textbf{Step 4:}  Output $\mathcal S$.
\end{algorithm}
\end{minipage}
\end{figure}

At a high level, this algorithm proceeds by dividing the space into cells, such that within each cell, the values of $\hat h$ are similar to each other. As a result, when $\hat h$ is close to some representation mapping $h$ of $p^*$ (that is, there exists a $w$, such that $w\circ h\approx p^*$ and $\hat h\approx h$), these sets can be assembled into the approximate level sets of $p^*$, on which, as we will show later, multi-calibration ensures accuracy. 
Moreover, the above algorithm returns $K=B^r$ sets in total, and guarantees that each set contains around $m_1/K$ samples. 
In Section~\ref{sec:accuracy} we will show that, under some mild model assumptions, $B\asymp m_1^{1/(r+\gamma)}$ (for some constant $\gamma>0$) will yield a minimax optimally accurate predictor. 

To show the validity of $\mathcal S$, we exhibit a simple post-processing step that yields a predictor that is asymptotically calibrated with respect to $\mathcal S$. This simplicity stems from the fact that the sets in $\mathcal S$ are pairwise disjoint. We denote the partition-identity function as $G:\mathcal X\to[K]$ where $G(x)=k$ if and only if $x\in\mathcal S_k$. 

Based on the the sets constructed in Algorithm \ref{alg:set}, we can further obtain $\hat{p}$ that is multi-calibrated on the sets $\cS$. Thus, we have the following algorithm.

\begin{figure}[ht]
  \centering
  \begin{minipage}{.9\linewidth}
\begin{algorithm}[H]
\caption{{\algfull}}
    \label{alg:full}
   
      \textbf{Input:} The representation $\hat h$, the number of partitions on each coordinate $B$, the dataset $D=\{(X_i,Y_i)\}_{i=1}^{m}$ 
      the splitting parameter $\pi$, and the calibration parameter $\alpha$
      
      \textbf{Step 1: }Set $m_1=\lfloor \pi m \rfloor$, $m_2=m-m_1$. We split the dataset $D$ into $D_1=\{(X_i^{(1)},Y_i^{(1)})\}_{i=1}^{m_1}$ and $D_2=\{(X_i^{(2)},Y_i^{(2)})\}_{i=1}^{m-m_1}$
      
   \textbf{Step 2:} Obtain $\cS$ through $D_1$ by applying Algorithm \ref{alg:set}.
   
\textbf{Step 3:} Use $D_2$ to obtain an $\alpha$-multi-calibrated predictor $\hat p$.

\textbf{Step 4:} Output $\hat p$.

\end{algorithm}
\end{minipage}
\end{figure}

\begin{Remark}
	In Step 3 of the Meta-Algorithm, we can apply any algorithm that produces a multi-calibrated predictor, {\it e.g.}, the iterative procedure proposed in \cite{HKRR}. Here, we propose and analyze a simple alternative method to obtain a  predictor $\hat p$ that is multi-calibrated with respect to $\mathcal S$ using the  data $D_2$: 
	\begin{equation}\label{eq:predictor}
		\hat p(x)=\frac{1}{|\{X^{(2)}_i\in S_{G(x)}\}|}\sum_{X^{(2)}_i\in S_{G(x)}} Y_i^{(2)}.
	\end{equation}
	Suppose we have a specific and moderately-sized collection $\mathcal C$ of sets for which we wish to ensure calibration; for example, these may be specific demographic groups.  Recall that we can not know when the \Scaffolding{} Set algorithm produces level sets.  When it does, multi-calibration produces a predictor that is close to $p^*$ (Theorem~\ref{thm:mc wrt S}), which (most likely) gives calibration with respect to $\mathcal C$.  Since we cannot be certain of success, we can still explicitly calibrate with respect to $\mathcal C$ using the algorithm of~\cite{HKRR} as a ``post-processing" step; \cite{HKRR} argues that such post-processing does not harm accuracy.
\end{Remark}

{{For the simplicity of presentation, throughout the paper, we set $m=2n$, $\pi=1/2$, and therefore $m_1=n$. Our results hold for any $\pi\in(c,1-c)$ with a universal constant $c>0$.}}
 The following theorem says that, assuming only that the training data are i.i.d., the resulting $\hat p$ is $\tilde O(\sqrt{K/n})$-multi-calibrated with respect to $\mathcal S$.
\begin{Theorem}\label{thm:mc}
Suppose we construct $\hat p$ as described in the Meta-Algorithm with Step 3 using Equation~\eqref{eq:predictor} (on the set $D$ of size $2n$) for any $B\in\nn$. Let $K=B^r$. If $D=\{(X_i, Y_i)\}_{i=1}^{2n}$ is $i.i.d.$ and $n\ge cK\log(K/\delta)$ for a universal constant $c>0$, then there exists a universal constant $C>0$, such that with probability at least $1-\delta-2\exp\left(\log K-{n/200K}\right)$ over $D$,
$$
|\bE_{X,Y}[Y-\hat p(X)\mid \hat p(X), X\in S_k]|\le C \sqrt{K\log(K/\delta)/n}, \text{ for all } k\in[K].
$$
\end{Theorem}

This result shows that, once we have our hands on $\hat{h}$, we can construct sets $\mathcal S$ and, with an extremely simple post-processing step, obtain a predictor (given by Equation~\eqref{eq:predictor}) that is asymptotically multi-calibrated with respect to $\mathcal S$ with high probability 
if $K\log K/n\to 0$. The condition $K\log K/n\to 0$ holds when we choose $B\asymp n^{1/(r+\gamma)}$.

 The proof of Theorem~\ref{thm:mc} relies on the fact that the $\hat p$ constructed by \eqref{eq:predictor} only takes finitely many values. Additionally, by our construction, the conditioned event $\{\hat p(x), X\in S_k\}$ has nontrivial mass and includes at least around $n/K$ samples. Therefore, even if we only assume $i.i.d.$ training data, with no assumption on the distribution, this property, combined with some concentration inequalities, implies the desired result on the calibration error.

\begin{proof}[Proof of Theorem~\ref{thm:mc}]
For a given $k\in[K]$, we define $N_k=|\{i\in[{ n}]:G(X_i^{(2)})=k \}|$, $\hat p_k=\sum_{X_i^{(2)}\in S_k} Y_i^{(2)}/N_k$ and $p_k^*=\E[Y\mid X\in S_k]$. 

\begin{Lemma}[Bernstein's inequality]
Let $X_1,...X_n$ be independent random variables such that $\E[X_i]=0$, $|X_i|\leq b$ and $\Var{(X_i)}\leq \sigma_i^2$ for all $i$. Let $\sigma^2=\sum_{i=1}^n\sigma_i^2/n$. Then for any $t>0$,
\begin{align*}
    \pp{\left(\frac{1}{n}|\sum_{i=1}^nX_i|\geq t\right)}\leq 2\exp\left(-\frac{nt^2/2}{\sigma^2+bt/3}\right).
\end{align*}
\end{Lemma}

\begin{Lemma}[Adapted from Lemma 4.3 in \cite{kumar2019verified}]\label{lem:radius}
Suppose we have i.i.d.\,samples from a distribution $\mathcal D$ over $\mathcal X$, $X_1,...,X_n\in\mathcal X$.  Suppose further that, for some $K>0$, we obtain a partition of the universe $\mathcal X=S_1\cup...\cup S_K$ such that $|\{i: X_i\in S_k\}|=n/K$ for all $k\in[K]$. Then there exists a universal constant $c$, such that if $n\ge cK\log(K/\delta)$, then for any $X$ independently drawn from $\mathcal D$, with probability at least $1-\delta$ over the choice of $\{X_i\}_{i=1}^n$, $$
\Prob_{X \sim \mathcal D}(X\in S_k)\in (1/2K,2/K), \text{ for all } k\in[K].
$$
\end{Lemma}

 We first prove the following claim showing that the differences $\hat p_k-p_k^*$, $k \in [K]$, are upper bounded by some stochastic upper bound uniformly for all $k$. For any $\delta\in(0,1)$, with probability at least $1-\delta$,\begin{equation}\label{ineq:concentration1}
 |\hat p_k-p_k^*|\le \sqrt{\frac{2\log(2K/\delta)}{N_k}}, \text{ for all } k\in[K].
  \end{equation}

In fact, for any realization of samples $(X_1^{(2)},...,X_{n_2}^{(2)})=(x_1,...,x_{n_2}):=x$, we define the event $E_{G(x)}=\{(G(X_1^{(2)}),...,G(X_{n_2}^{(2)}))=(G(x_1),...,G(x_{n_2}))\}$. On event $E_{G(x)}$, we have $N_k$'s fixed and in each bin $\{i: X_i^{(2)}\in S_k\}$ $Y_i$'s are $i.i.d.\sim Ber(p_k^*)$. As a result, using Bernstein's inequality and the fact that $p_k^*\le 1$, we have$$
\Prob\left(|\hat p_k-p_k^*|\ge t\mid E_{G(x)}\right)\le  2\exp(-\frac{N_k t^2}{{2 p_k^*}+2t/3})\le 2\exp(-\frac{N_k t^2}{{2}+2t/3}).
$$
Using the assumption $n\ge cK\log(K/\delta)$ for a universal constant $c>0$,  then there exists a universal constant $C_1$ such that when we take  $t=\sqrt {C_1\log(K/\delta)/N_k}$, we have 
$$
\Prob(|\hat p_k-p_k^*|\ge \sqrt\frac{C_1 \log(K/\delta)}{N_k}\mid E_{G(x)})\le  2\exp(-\log(2K/\delta))=\delta/K,
$$
implying $$
\Prob(|\hat p_k-p_k^*|\ge \sqrt\frac{C_1\log(K/\delta)}{N_k})\le  \delta/K.
$$
By a union bound, we then obtain with probability at least $1-\delta$,\begin{equation*}
 |\hat p_k-p_k^*|\le \sqrt{\frac{C_1\log(K/\delta)}{N_k}}, \text{ for all } k\in[K].
  \end{equation*}

We now use \eqref{ineq:concentration1} to derive the concentration inequalities for $N_k$. 

Let $q_k=\Pro(X_i\in S_k)$.  By Bernstein's Inequality we have $$
\pp(|N_k/n_2-q_k|\ge t)\leq 2\exp\left(-\frac{n_2t^2/2}{q_k+2t/3}\right).
$$
We take $t=1/4K$, then
$$
\pp(|N_k/n_2-q_k|\ge 1/4K)\leq 2\exp\left(-\frac{n_2/32K^2}{q_k+1/6K}\right).
$$
If we have 
$q_k\in (1/2K,2/K)$, then
$$
\pp(N_k/n_2 \le 1/4K)\leq 2\exp\left(-{n_2/100K}\right),
$$
implying $$
\pp(\min_{k\in[K]}N_k/n_2 \ge 1/4K)\geq 1- 2\exp\left(\log K-{n_2/100K}\right)\geq 1- 2\exp\left(\log K-{n/200K}\right).
$$
Combining with the inequality~\ref{ineq:concentration1}, we have with probability at least $1-\delta-2\exp\left(\log K-{n/200K}\right)$,$$
|\hat p_k-p_k^*|\le \sqrt{\frac{4C_1\log(K/\delta)}{n_2/K}}\le  \sqrt{\frac{8C_1\log(K/\delta)}{n/K}}, \text{ for all } k\in[K].
$$

\end{proof}

\subsubsection{Accuracy Analysis}
\label{sec:accuracy}
In this subsection, we obtain conditions on $\hat h$ under which multi-calibration with respect to the collection $\mathcal S$ of \scaffolding{} sets produced by Algorithm~\ref{alg:set} yields accuracy guarantees, specifically, closeness to $p^*$.  

As usual, we assume that we have $i.i.d.$ observations $D=\{( X_i,Y_i)\}_{i=1}^{2n}$  drawn from a distribution $\cD$ satisfying \begin{equation}\label{eq:model}
\E[Y_i\mid X_i]=p^*(X_i), \text{ for } i=1,2,...,2n,
\end{equation}
where $p^*(X): \R^d\to \R$ denotes the true probability (in the classification setting)\footnote{Our analysis can be extended to the regression case with slight modifications.}.

For $\beta\in(0,1]$ and $L>0$, let
 $\mathcal F(\beta,L)=\{w: \R^r\to\R\mid |w(h)-w(h')|\le L \|h-h'\|^\beta \text{ for any } h,h'\in\R^r\}$
 denote the class of all $(\beta,L)$-H\"older smooth function.

Let $H(p^*,\beta,L)$ be defined as $$
H(p^*,\beta,L)=\{h: \exists w \in \mathcal F(\beta,L), \text{ s.t. }p^*(x)=w \circ {h}(x)\}.
$$

\begin{Assumption}
  \label{assumption on h}
There exists an $h\in H(p^*,\beta,L)$, such that $\ee_X[\|\hat h(X)-h(X)\|^2]=e_h$ for some $e_h>0$, and for each $j$, $\hat h_j(x)$ has bounded and positive probability density on a compact set within $[-C,C]$. 
\end{Assumption}

\begin{Theorem}
  \label{thm:mc wrt S}
  $\exists c \in \rr$ such that, $\forall r, B \in \mathbb{N^+}$, if $\hat{h}$ satisfies Assumption~\ref{assumption on h}, then, letting $\mathcal S$ denote the output of the \scaffolding{} set construction algorithm when running on $\hat h$ and a training dataset of size $n\ge cB\log(rB)$ 
  then the following statement holds.
If $\hat p$ is $\alpha_n$-multi-calibrated with respect to $\mathcal S$, then for $X\sim\cD$
$$
\ee_{X,D}[(\hat p(X)-p^*(X))^2]\lesssim L^2(\frac{r^\beta }{B^{2\beta}}+e_h^\beta)+\alpha_n^2.
$$
Here we use the notation $\ee_{X,D}$ to denote that we take expectation over the randomness of $X$ and $\hat p$, as $\hat p$ is constructed based on $D$.
\end{Theorem}

{The proof of Theorem~\ref{thm:mc wrt S} makes use of the observation that, if the collection $\mathcal S$ contains the (approximate) level sets of $p^*$, then multi-calibration with respect to $\mathcal S$ ensures closeness to~$p^*$. By the  smoothness condition on $w$, two points that are close in $h$ will also be close in $p^*$, and therefore the approximate level sets of $h$ (found by using the quantiles in Step 2 of Algorithm~\ref{alg:set}) can be used to construct the approximate level sets of $p^*$. This eventually yields an accuracy guarantee. }

\begin{proof}

Fix an arbitrary set $S_k^{(\hat h)} \in \mathcal S$.  For all  $x\in S_k^{(\hat h)}$, we first define
$\lambda_k(X):=\bE_{X'}[w\circ h(X')|\hat{p}(X')=\hat{p}(X), X'\in S^{(\hat{h})}_k]$, where $X'$ is independent of $X$ and has the same distribution as $X$. 
By telescoping, we can obtain
\begin{align*}
\bE_{X,D}[(\hat{p}(X)-p^{*}(X))^2]=&\sum_k\bE_{X,D}[(\hat{p}(X)-p^{*}(X))^2 I\{X\in S^{(\hat h)}_k\}]  \\
\le& 2\sum_k\bE_{X,D}[(\hat{p}(X)-\lambda_k(X))^2 I\{X\in S^{(\hat h)}_k\}] \\ &+2\sum_k\bE_{X,D}[(p^{*}(X)-\lambda_k(X))^2 I\{X\in S^{(\hat h)}_k\}].
\end{align*}

Now, let us look into the term $\hat{p}(X)-\lambda_k(X)$. Since we know that $\hat p(\cdot)$ is $\alpha_n$ multi-calibrated on $S^{(\hat h)}_k$, Thus, 
\begin{align*}
|\hat p(X) - \lambda_k(X)|&=|\hat p(X)-\bE_{X'}[w\circ h(X')|\hat{p}(X')=\hat{p}(X), X'\in S^{(\hat{h})}_k]| \le \alpha_n.
\end{align*}

As a result, we have 
$$ 2\sum_k\bE_{X,D}[(\hat{p}(X)-\lambda_k(X))^2 I\{X\in S^{(\hat h)}_k\}]\le 2 \sum_k\alpha_n^2 \bP(X\in S^{(\hat h)}_k) = 2\alpha^2_n.$$

On the other hand, 
\begin{align*}
&\bE_{X,D}[(p^{*}(X)-\lambda_k(X))^2 I\{X\in S^{(\hat h)}_k\}]\\
\lesssim& \bE_{X,D} [(w\circ \hat h(X)-\bE_{X'}[w\circ \hat h(X')|\hat{p}(X')=\hat{p}(X), X'\in S^{(\hat{h})}_k])^2I\{X\in S^{(\hat h)}_k\}]\\\
&+ \bE_{X,D}[(w\circ (h(X)-\hat h(X)))^2I\{X\in S^{(\hat h)}_k\}]\\
&+\bE_{X,D}\left((\bE_{X'}[w\circ h(X') - w\circ \hat h(X')|\hat{p}(X')=\hat{p}(X), X'\in S^{(\hat{h})}_k])^2I\{X\in S^{(\hat h)}_k\}\right).
\end{align*}

Meanwhile, note that 
\begin{align*}
    &\bE_{X,D}\left((\bE_{X'}[w\circ h(X') - w\circ \hat h(X')|\hat{p}(X')=\hat{p}(X), X'\in S^{(\hat{h})}_k])^2I\{X\in S^{(\hat h)}_k\}\right)\\
    \le&  \bE_{X,D}\left((\bE_{X'}[(w\circ h(X') - w\circ \hat h(X'))^2|\hat{p}(X')=\hat{p}(X), X'\in S^{(\hat{h})}_k])I\{X\in S^{(\hat h)}_k\}\right)\\
     \le& L^2 \bE_{X,D}\left((\bE_{X'}[(\| h(X') -  \hat h(X')\|^{2\beta}I\{X\in S^{(\hat h)}_k\}|\hat{p}(X')=\hat{p}(X), X'\in S^{(\hat{h})}_k])\right)\\
      =& L^2(\bE_{X'}[(\| h(X') -  \hat h(X')\|^{2\beta}|X'\in S^{(\hat{h})}_k]\bP(X\in S^{(\hat{h})}_k))=L^2(\bE_{X'}[(\| h(X') -  \hat h(X')\|^{2\beta}|x'\in S^{(\hat{h})}_k]\bP(X'\in S^{(\hat{h})}_k))\\
      =&L^2(\bE_{X'}[\| h(X') -  \hat h(X')\|^{2\beta}\cdot I\{X'\in S^{(\hat{h})}_k\}].
\end{align*}
For the second inequality, since the indicator $I\{X\in S^{(\hat h)}_k\}$ can be viewed as a real number (no randomness) before taking expectation over $X$, so we can move the indicator function.

As a result, the last two terms, by Assumption 1, are both bounded by $L^2e_h^\beta$. For the remaining term $$\bE_{X,D}[(w\circ \hat h(X)-\bE_{X'}[w\circ \hat h(X')|\hat{p}(X')=\hat{p}(X), X'\in S^{(\hat{h})}_k])^2I\{X\in S^{(\hat h)}_k\}].$$

Since both $X$ and $X'$ are in $S^{(\hat h)}_k$, using Lemma~\ref{lem:radius} and combining with the assumption that for each $j$, $\hat h_j(X)$ has bounded and positive probability density on a compact set in $[-C,C]$, we have that the length of each coordinate of $S_k^{\hat h}$ is $\Theta(1/B)$ uniformly for all coordinates simultaneously with probability $1-B^{-2}$ when $n\ge cB\log(rB)$ for some universal constant $c$.

As a result, the radius of each cell $S^{(\hat h)}_k$ is $\Theta(
\sqrt{r}/B)$ with probability $1-B^{-2}$. Thus, we know that by Holder smoothness, that 
$$\bE_{X,D}[(w\circ \hat h(X)-\bE_{X'}[w\circ \hat h(X')|\hat{p}(X')=\hat{p}(X), X'\in S^{(\hat{h})}_k])^2I\{X\in S^{(\hat h)}_k\}]\lesssim L^2\frac{r^\beta}{B^{2\beta}}+\frac{2}{B^2}\lesssim L^2\frac{r^\beta}{B^{2\beta}}.$$

Combining all the terms together, the proof is complete.
\end{proof}
\begin{Remark}
We note here that Assumption~\ref{assumption on h} can be relaxed to allow approximation error, that is, $h\in H(p^*,\beta,L,\epsilon)$, where 
$H(p^*,\beta,L,\epsilon)=\{h: \exists w \in \mathcal F(\beta,L), s.t. \sup_{x\in\mathcal X}|p^*(x)-w \circ {h}(x)|\le\epsilon\}.
$ In this case a proof similar to that of Theorem~\ref{thm:mc wrt S} yields $
\ee_{X,D}[(\hat p(X)-p^*(X))^2]\lesssim L^2(\frac{r^\beta }{B^{2\beta}}+e_h^\beta)+\alpha_n^2+\epsilon^2.
$
\end{Remark}

This theorem further implies that for the predictor $\hat p$ constructed as in~\eqref{eq:predictor}, we have the following more refined accuracy guarantee. 
\begin{Corollary}\label{col:1}
Suppose we construct $\hat p$ as described in~\eqref{eq:predictor}.
Then, under Assumption~\ref{assumption on h}, we have $$
\ee_{X,D}[(\hat p(X)-p^*(X))^2]=\tilde O( L^2(\frac{r^\beta }{B^{2\beta}}+e_h^\beta)+\frac{B^r}{n}).
$$
In particular, if we choose $B\asymp L^{2/(r+2\beta)}n^{1/(r+2\beta)}$, we will then have $$
\ee_{X,D}[(\hat p(X)-p^*(X))^2]=\tilde O(L^2 e_h^\beta+L^{2r/(r+2\beta)}n^{-2\beta/(r+2\beta)}).
$$
\end{Corollary}

\begin{Remark}
\label{rem:learnability}

In practice, we can also further split $D_1$ equally and try to use the first half to {\em learn} an $\hat h$.  
The hope in this scenario is that more training data will lead to   learned functions $\hat h$ that, in the context of Assumption~\ref{assumption on h}, will have error terms $e_h$ that vanish as the sample size $n\to \infty$. We will refer to this as  the {\em Learnability Assumption.}  
Note that, as $n$ grows, yielding a sequence of functions, $\hat{h}_n$, the $h_n$'s that witness the fact that the $\hat h_n$'s satisfy Assumption~\ref{assumption on h} may change with~$n$.

Under the Learnability Assumption, the corollary implies that, if $e_h=\tilde O( L^{-4/(r+2\beta)}n^{-2/(r+2\beta)})$, then $$\ee_{X,D}[(\hat p(X)-p^*(X))^2]=\tilde O(L^{2r/(r+2\beta)}n^{-2\beta/(r+2\beta)}).
$$
In fact, this is the best rate one could achieve since we have the following minimax lower bound \cite{tsybakov2009introduction}: there exists a universal constant $C$, such that $$
\inf_{\hat p}\sup_{p^*=w\circ h: w\in\mathcal F(\beta,L)}\ee_{X,D}[(\hat p(X)-p^*(X))^2]\geq C\cdot L^{2r/(r+2\beta)}n^{-2\beta/(r+2\beta)}.
$$
\end{Remark}
This minimax result demonstrates that for any proposed estimator $\hat p$ obtained by training, there exists $p^*$ belonging to the class of structures we considered, the discrepancy between $\hat p$ and $p^*$ has to be larger than the rate $L^{2r/(r+2\beta)}n^{-2\beta/(r+2\beta)}$; in other words, this rate is the best one can get.
\begin{Remark}
In Section~\ref{sec: learning h}, we will present several methods that produce $e_h$ satisfying the condition $$e_h=\tilde O( L^{-4/(r+2\beta)}n^{-2/(r+2\beta)}),$$ when $p^*$ is expressible by different types of neural networks.
\end{Remark}

\subsubsection{``No Harm"}
In this section, we show that even if Assumption~\ref{assumption on h} fails to hold, running the Meta-Algorithm with Equation~\eqref{eq:predictor} for the multi-calibration step
will not harm accuracy. Specifically, suppose that we already have a prediction function $\hat p_0=\hat w\circ\hat h$ (from any training algorithm). We show that if 
we post-process ${\hat p}_0$ by using $\hat h$ as an input of Algorithm~\ref{alg:full}, and compute the Step 3 there via Equation~\eqref{eq:predictor}, then
the resulting predictor is no less accurate than $\hat p_0$.
It is in this sense that Algorithm~\ref{alg:full} ``does no harm".
This may be viewed as an analogue to the results of~\cite{HKRR} that multi-calibrating an existing predictor does not harm accuracy.

Recall that in Equation~\eqref{eq:predictor}, on each set $S$ that is found by the \Scaffolding{} Set algorithm (Algorithm \ref{alg:set}), we set $\hat p$ to 
$\hat p(x)=\sum_{X_i\in S}Y_i/|S|. 
$
We have the following proposition. 

\begin{Proposition}\label{prop:no harm} Suppose that we have $\hat w\in\mathcal F(\beta,L)$, $\hat h(x)\in[-C,C]^r$, and $\hat h(X)$ has continuous and positive probability density over $[-C,C]^r$.
Then there exists a constant $C$ such that, for $\hat p$ obtained by running the Meta-Algorithm (Algorithm~\ref{alg:full}) with multi-calibration achieved via Equation \ref{eq:predictor} on input $\hat h$, training set $D$ of size $2n$, $B$ and $r$
$$\ee_{X,D}[(p^*(X)-\hat p(X))^2]\le\ee_X[(p^*(X)-\hat p_0(X))^2]+\tilde O(L^2\frac{r^\beta}{B^{2\beta}}+\frac{B^r}{n}).$$
\end{Proposition}
\begin{Remark}
Note that the conditions we impose are on $\hat w$ and $\hat h$, which can be enforced during the training process. For example, we can enforce the smoothness of $\hat w$ by imposing bounded norm constraints on the parameters in neural network training. 
\end{Remark} 

This ``no harm'' result is proved using a decomposition of the mean squared error (MSE) $\ee_{X,D}[(p^*(X)-\hat p(X))^2]=\sum_{S\in\mathcal S}\ee_{X\in S,D}[(p^*(X)-\hat p(X))^2]$. In each $S$, the values of $\hat p$ vary little, and are almost constant. Using the fact that the mean minimizes the MSE, replacing this constant with $\E_{X\in S}[p^*(X)]$ therefore yields a smaller MSE. 

\begin{proof}
We first analyze this problem at the population level (when we have infinite number of samples). The \Scaffolding{} Set construction algorithm yields a collection of sets. For any set $S$ in this collection, and for any $x\in S$, at the population level, we update $\hat p_0$ to obtain $$
\tilde p(x)=\frac{1}{\mu_X(S)}\int_{x\in S} p^*(x)\;d\mu(x),
$$
where $\mu_X(S)=\int_S \;d\mu(x)$. We also define, for $x\in S$, $\bar p_0(x)=\frac{1}{\mu_X(S)}\int_{x\in S} \hat p_0(x)\;d\mu(x)$. 
We present the following lemma 
\begin{Lemma}\label{lm:pop} Using the update rule described above, we have 
$$\int_{x\in S}(p^*(x)-\tilde p(x))^2 \;d\mu(x)\le\int_{x\in S}(p^*(x)-\bar p_0(x))^2 \;d\mu(x).$$
\end{Lemma}
\begin{proof}
Let $f(c)=\int_{x\in S}(p^*(x)-c)^2 \;d\mu(x)$. Then we have $$
f'(c)=2\int_{x\in S}(c-p^*(x)) \;d\mu(x).
$$
Setting $f'(c)=0$, we get $f$ obtains the minimum when $c=\frac{1}{\int_S \;d\mu(x)}\int_{x\in S} p^*(x)\;d\mu(x)$.
\end{proof} 
\paragraph{[Proof of Proposition \ref{prop:no harm}].} Let $\tilde p$ be defined as $\frac{1}{\mu_X(S)}\int_{x\in S} p^*(x)\;d\mu(x)$ and $\bar p_0=\frac{1}{\mu_X(S)}\int_{x\in S} \hat p_0(x)\;d\mu(x)$. By Lemma~\ref{lm:pop}, we have $$
\int_{x\in S}(p^*(x)-\tilde p(x))^2 \;d\mu(x)\le\int_{x\in S}(p^*(x)-\bar p_0(x))^2 \;d\mu(x).
$$
So all we need now is to upper bound 
$
\int_{x\in S}(\hat p(x)-\tilde p(x))^2 \;d\mu(x)
$ and $
\int_{x\in S}(\hat p_0(x)-\bar p_0(x))^2 \;d\mu(x).
$ 
The first term is due to the finite sample. 
By definition, recall that we have $\hat p(x)=\frac{1}{\#\{X_i\in S\}}\sum_{X_i\in S}Y_i$, and $\tilde p(x)=\frac{1}{\mu_X(S)}\int_{x\in S} p^*(x)\;d\mu(x)$. 
By Theorem~\ref{thm:mc}, with probability at least $1-2\delta$ over $D$, we have $$
\ee_x \|\tilde p(x)-\hat p(x)\|^2\le \frac{B^r\log n}{n}.
$$
For the second term, using the result in the proof of Theorem~\ref{thm:mc wrt S}, we have that with probability at least $1-B^{-2}$, the radius of each cell $S^{(\hat h)}_k$ is $\Theta(
\sqrt{r}/B)$, so
$|\hat h(x)-\hat h(x')|\lesssim\frac{\sqrt r}{B}$, we then have $|\hat p_0(x)-\hat p_0(x')|\lesssim L(\frac{\sqrt r}{B})^\beta$. As a result, we have $$
\int_{x\in S}(\hat p_0(x)-\bar p_0(x))^2 \;d\mu(x)\le L^2(\frac{\sqrt r}{B})^{2\beta}\cdot\mu_X(S).
$$
Combining all the pieces and summing up on all $S$'s, we obtain the desired result.

\end{proof}

\section{Methods for Finding $\hat h$}
\label{sec: learning h}
In the previous section, the accuracy guarantees (Theorem 2) rely on the assumption that the input $\hat h$ is close to some function $h$ which can serve as a representation mapping for $p^*$ (Assumption 1).
In this section, we present methods for finding such an~$\hat h$ in the neural network setting, as well as an extension to the transfer learning setting. 

\subsection{Obtaining $\hat h$ in a homogeneous neural network}
Consider the binary classification problem where we have $n$ i.i.d.~observations $\{(X_i,Y_i)\}_{i \in [n]}$ such that $\E[Y_i\mid X_i]=p^*(X_i)$, and in which $p^*$ has the form 
\begin{equation}\label{eq:model fitting}
p^*(x)=\Pr[Y=1\mid X=x]=W_k(\sigma(W_{k-1}\sigma(...\sigma(W_1 x))),
\end{equation}
where $x\in\R^{d_1}$, $W_1\in\R^{1\times d_1}$, $W_2\in\R^{d_2\times d_1}, W_j\in\R^{d_j\times d_{j-1}}$ for $j=3,...,k-1$, $W_k\in\R^{1\times d_{k-1}}$, and $\sigma(\cdot)$ is the ReLU activation function $\sigma(x)=\max\{0,x\}$. Suppose 
$\|W_j\|_2=O(1)$, for $j=1,2,...,k$. Such a neural network model is said to be {\em homogeneous} because multiplying $x$ by a positive number $a$ will result in multiplying the output by $a^m$, for some $m\in\bN$. Such a model approximates the single index model, which is a commonly used data distribution assumption in economics \cite{powell1989semiparametric, horowitz2009semiparametric}, time series \cite{fan2008nonlinear}, and survival analysis \cite{lu2006class}. The model \eqref{eq:model fitting} is also commonly used in the theoretical deep learning community, see \cite{ge2018learning, lyu2019gradient, ge2019learning}.

Suppose that we now solve 
$$
\hat W_1=\arg\min_{W_1}\frac{1}{n}\sum_{i=1}^n (Y_i-W_1 X_i)^2,
$$
and then set $\hat h(x)=\hat W_1x$.
The following theorem shows that this learned $\hat h$ is a good representation mapping for a rich class of data distributions. 

\begin{Theorem}\label{thm:ReLU:rep}
Suppose $X_i$'s are i.i.d.~drawn from a symmetric \footnote{For a distribution with probability density $p$, we say this distribution is symmetric if and only if $p(x)=p(-x)$ for all $x$.} and sub-gaussian distribution with covariance matrix $\Sigma_X$ 
{and} positive densities on a compact support within $\{x\in\R^{d_1}:\|x\|\le C\}$ for some constant $C>0$. We also assume $c_1\le\lambda_{\min}(\Sigma_X)\le\lambda_{\max}(\Sigma_X)\le c_2$ for some universal constants $c_1,c_2$. Moreover, we denote $q(v)=W_k(\sigma(W_{k-1}(...\sigma(W_2v))$ and assume $q\not\equiv 0$. Letting $\gamma=\frac{1}{2}q(1)$ and  $h(x)=\gamma W_1 x$, we then have with probability at least $1-n^{-2}$, $$
\ee_{X}\|\hat h(X)-h(X)\|^2=O(\frac{d_1}{n}).
$$
\end{Theorem}
\remove{Analogous statements to Theorems~\ref{thm:sigmoid} and~\ref{thm:ReLU} hold for {\em any} $\hat p$ that is multi-calibrated with respect to the \scaffolding{} sets output by Algorithm~\ref{alg:part}; the key difference is the calibration error $\alpha_n$, as in the bound given in Theorem~\ref{thm:mc wrt S}.}

We remark that Theorem~\ref{thm:ReLU:rep} shows that $\hat h(x)$ is close to $h(x)=\gamma W_1 x$. Now, letting $w(h)=W_k(\sigma(W_{k-1}(...\sigma(h/\gamma))$, we have $w\circ h=p^*$. Using the set construction algorithm proposed above with input $\hat h(x)=\hat W_1x$, we then obtain the following corollary regarding the accuracy produced by multi-calibrating with respect to the output $\mathcal S$. 
\begin{Corollary}\label{cor:ReLU}
Under the same conditions as in the statement of Theorem~\ref{thm:ReLU:rep}, and further assuming that the depth $k\le C_1$ and $\|W_j\|_2\le C_2$, for $j=1,2,...,k$ for some universal constants $C_1, C_2$, we have$$
\ee_{X,D}[(p^*(X)-\hat p(X))^2]\lesssim \frac{d_1}{n}+ (\frac{\log n}{n^{1/3}})^2+\alpha_n^2.
$$

\end{Corollary}
\begin{Remark}
  The sub-Gaussian distribution is quite general in machine learning. For example, any data where the feature values are bounded are sub-Gaussian. This includes all image data since the pixel values are bounded. 
\end{Remark}
\begin{Remark} We can interpret the symmetry requirement in two ways. First, in the practice of training neural networks, people often pre-process the data to make them centralized (a process also known as {\em standardization}). The symmetry requirement then becomes ``symmetric around the mean". Second, we can also augment the dataset to enforce it to satisfy this requirement, that is, for a sample $x$, we add a $-x$ to the training data. In image data, this augmentation corresponds to adding \textit{negative images} in photography.
  \end{Remark}
  
  \begin{proof}[Proofs of Theorems~\ref{thm:ReLU:rep} and Corollary~\ref{cor:ReLU}]\label{sec:proof:relu}
We first present a handy lemma.

\begin{Lemma}\label{lem:relu}
Suppose $X_i$'s are i.i.d. sampled from a symmetric and sub-gaussian distribution, and denote $q(v)=W_k(\sigma(W_{k-1}(...\sigma(W_2v))$ with $\sigma$ being the ReLU function $\sigma(x)=\max\{x,0\}$, then we have
$$\E[Y_iX_i]=\E[p^*(X_i)X_i]=\frac{1}{2}q(1)\cdot\Sigma_X W_1^\top.$$
\end{Lemma}
The proof of Lemma~\ref{lem:relu} is deferred to the appendix.

Now let us recall that $\gamma=\frac{1}{2}q(1)$.
Since $\|W_j\|\lesssim 1$ for $j=1,2,...,k$, we have $g$ is  $L$-Lipschitz for some universal constant $L$, which also implies $\gamma\le L$. 

We then analyze the convergence of $\hat W_1$. By definition, we have $$
\hat W_1=(\frac{1}{n}\sum_{i=1}^n X_iX_i^\top)^{-1}(\frac{1}{n}\sum_{i=1}^n X_iY_i).
$$
Using standard concentration inequality for Wishart matrices for sub-gaussian distribution, we have with probability at least $1-n^{-2}$, $$
\|\frac{1}{n}\sum_{i=1}^n X_iX_i^\top-\Sigma_X\|=\|\frac{1}{n}\sum_{i=1}^n X_iX_i^\top-\E[X_iX_i^\top]\|=O(\sqrt\frac{d_1}{n}),
$$
and $$
\|\frac{1}{n}\sum_{i=1}^n X_i Y_i-\E[X_iY_i]\|=O(\sqrt\frac{d_1}{n}).
$$
Then using Lemma \ref{lem:relu}, suppose $\lambda_{\min}(\Sigma)\gtrsim c_1$ and $d_1/n\le c_2$ for some universal constants $c_1$ and $c_2$, we then have $$
\|\hat W_1-\gamma W_1\|=\|(\frac{1}{n}\sum_{i=1}^n X_iX_i^\top)^{-1}(\frac{1}{n}\sum_{i=1}^n X_iY_i)-\Sigma_X^{-1}\E[X_iY_i]\|=O_P(\sqrt\frac{d_1}{n}).
$$
We now let $\hat h(x)=\hat W_1x$, $h(x)=\gamma W_1 x$. We then have with probability at least $1-n^{-2}$, $$
\ee_x(\hat h(x)-h(x))^2= O(\frac{d_1}{n}).
$$

Recall that $w(x)=W_3\sigma(W_2\sigma(u/\gamma))$. We then have $w$ is $1$-Lipschitz, and using Theorem~\ref{thm:mc wrt S}, we obtain $$
\E[(p^*(x)-\hat p(x))^2]\lesssim \frac{d_1}{n}+ (\frac{\log n}{n^{1/3}})^2+\alpha_n^2.
$$
  
  \end{proof}
  
\subsection{Obtaining $\hat h$ in an inhomogeneous neural network}
In the this section, we present a theorem showing that if we assume the input distribution is Gaussian, then we are able to allow $p^*$ to be an inhomogeneous neural network and include bias terms, that is, 
\begin{equation}\label{eq:model fitting2}
p^*(x)=W_k(\sigma(W_{k-1}(\sigma(...\sigma(W_1 x+b_1))+b_{k-2})+b_{k-1}),
\end{equation}
where $x\in\R^{d_1}$, $W_1\in\R^{1\times d_1}$, $W_2\in\R^{d_2\times 1}, W_j\in\R^{d_j\times d_{j-1}}$ for $j=3,...,k-1$, $W_k\in\R^{d_k\times 1}$; $b_1\in\R, b_k\in\R^{d_k}$, and $\sigma(\cdot)$ being a general activation function.

\begin{Theorem}\label{thm:sigmoid:rep}
Suppose $X_i$'s are i.i.d.\,drawn from a Gaussian distribution $N_{d_1}(0,\Sigma_X)$. We also assume $c_1\le\lambda_{\min}(\Sigma_X)\le\lambda_{\max}(\Sigma_X)\le c_2$ for some universal constants $c_1,c_2$. Moreover, we denote 
$$g(u)=W_k(\sigma(W_{k-1}(\sigma(...\sigma(u+b_1))+b_{k-2})+b_{k-1})$$ 
and assume $g$ is $L$-Lipschitz with $L<C$ for some constant $C>0$, and  $\gamma=\E[g'(W_1x)]\neq 0$. Letting  $h(x)=\gamma W_1 x$, we then have with probability at least $1-n^{-2}$, $$
\ee_{X}\|\hat h(X)-h(X)\|^2=O(\frac{d_1}{n}).
$$
\end{Theorem}
\begin{Remark}
 Note that by taking $w(h)=W_k(\sigma(W_{k-1}(...\sigma(h/\gamma+b_1))+b_{k-2})+b_{k-1})$, we will have $w\circ h=p^*$. A bound on the MSE for $p^*$ analogous to that of Corollary~\ref{cor:ReLU} can be obtained by using Theorem~\ref{thm:sigmoid:rep} with slight modifications in the proof.
\end{Remark}

\begin{proof}[Proof of Theorem~\ref{thm:sigmoid:rep}]
The proof of theorem~\ref{thm:sigmoid:rep} follows the same strategy of Theorem~\ref{thm:ReLU:rep}. We first present an analogue to Lemma~\ref{lem:relu}.

\begin{Lemma}\label{lem:sigmoid}
Suppose $X_i$'s are i.i.d.\,samples from $N_{d_1}(0,\Sigma_X)$ and denote $g(u)=W_k(\sigma(W_{k-1}(\sigma(...\sigma(u+b_1))+b_{k-2})+b_{k-1})$, with $\sigma$ being the a general activation function, Then
$$\E[Y_iX_i]=\E[p^*(X_i)X_i]=\E[g'(W_1x)]\Sigma_X W_1^\top.$$
\end{Lemma}

Then following the exact same analysis in Section~\ref{sec:proof:relu}, we obtain the desired result.

\noindent\textbf{Proof of Lemma~\ref{lem:sigmoid}:}
We will use the First-order Stein's Identity \cite{diaconis2004use}.
\begin{Lemma}[First-order Stein's Identity\cite{diaconis2004use}]\label{lm:stein}
Let $X \in \R^d$ be a real-valued random vector with density $\rho$. Assume that $\rho$: $\R^d \to R$ is differentiable. In addition, let $g : \R^d \to \R$ be a continuous function such that $\E[\nabla g(X)]$ exists. Then it holds that$$
\bE_{{X\sim \rho}}[g(X)\cdot S(X)]=\E[\nabla g(X)],
$$
where $S(X)=-\nabla \rho(x)/\rho(x)$. 
\end{Lemma}
Now, let us plug in the density of $N_{d_1}(0,\Sigma_X)$, $p(x)=ce^{x^\top\Sigma_X^{-1}x/2}$ for some constant $c$. We then have $\nabla p(x)=ce^{x^\top\Sigma_X^{-1}x/2}\cdot \Sigma_X^{-1}x$ and $\nabla p(x)/p(x)= \Sigma_X^{-1} x$.

As a result, we have $$
\E[p^*(x) \Sigma_{X}^{-1}x]=\E[\nabla p^*(x)],
$$
implying $$
\E[p^*(x)x]= \Sigma_{X}\E[\nabla p^*(x)].
$$
Then recall that $p^*(x)=g(W_1x)$, so we have $\nabla p^*(x)=g'(W_1x)W_1^\top$. Combining all the pieces, we obtain
$$
\E[p^*(x)x]= \Sigma_{X}\E[g'(W_1x)]W_1^\top.
$$
\end{proof}

\subsection{Obtaining $\hat h$ through transfer learning}

In this section, we present an example of obtaining $\hat h$ in the transfer learning setting \cite{du2020few,tripuraneni2021provable,deng2021adversarial}, where in addition to the samples from the target model, we have auxiliary samples from different but possibly related models. Following the terms in the transfer learning literature, we will refer to the target model as the {\em target task}, and the  auxiliary models as the {\em source tasks}. 

In transfer learning, there are two principal settings: (1) {\em covariate shift}, where the marginal distributions of $X$ are assumed to be different across different source tasks; and (2) {\em concept shift}, where the conditional distribution $\Pro(Y\mid X)$ are different. We will consider both types of shift. 

\paragraph{Covariate shift.}
In the setting in which the marginal distribution on $X$ vary across source tasks, we assume that there are $T$ source tasks, and in the $t$-th source task ($t\in[T]$), we observe $n$ $i.i.d.$ samples $(X_i^{(t)}, Y_i^{(t)})$ from the model $Y^{(t)}_i\sim Ber(p^*(X^{(t)}_i))$, where $p^*$ has the form 
\begin{equation}\label{eq:model fitting-2}
p^*(x)=\Pr[Y=1\mid X=x]=W_k(\sigma(W_{k-1}(\sigma(...\sigma(W_1 x+b_1))+b_{k-2})+b_{k-1}),
\end{equation}
where $x\in\R^{d_1}$, $W_1\in\R^{r\times d_1}$, $W_2\in\R^{d_2\times r}, W_j\in\R^{d_j\times d_{j-1}}$ for $j=3,...,k-1$, $W_k\in\R^{1\times d_{k-1}}$, $b_1\in\R^r, b_k\in\R^{d_k}$, and $\sigma(\cdot)$ is a general activation function. We assume $r<d_1$, and note here in this model we treat $W_1 x$ as the representation function $h(x)$  in Section~\ref{sec:main_results}. 
{Note that the representation function $h$ is not unique.
That is, let $h(x)=W_1x$; and let $w(u)= W_k(\sigma(W_{k-1}(\sigma(...\sigma(u+b_1))+b_{k-2})+b_{k-1})$.  Then $p^*$ in \eqref{eq:model fitting-2} satisfies $p^*=w\circ h$; however, for any invertible matrix $Q\in\R^{r\times r}$, we can also find $\tilde w(u)=w(Q^{-1}u)$ and $\tilde h(x)=QW_1x$ such that $p^*=w\circ h=\tilde w\circ \tilde h$.  Therefore, for simplicity of presentation, 
we assume $W_1$ is an orthogonal matrix such that $W_1W_1^\top=I_{r}$ (for an arbitrary matrix $W_1\in\R^{r\times d_1}$, the existence of the matrix $Q$ such that $QW_1$ is an orthogonal matrix is guaranteed by the singular value decomposition (SVD) if we assume $\lambda_{r}(W_1)>0$). }

We now study how to learn {an approximation to} $W_1$. 
For each source task, we fit the data by a (smaller) linear model $$
\hat\beta^{(t)}=\arg\min\frac{1}{n}\sum_{i=1}^n (Y_i^{(t)}-\beta^\top X_i^{(t)})^2.
$$
We can then learn $W_1$ by performing SVD on the matrix $\hat B=[\hat\beta^{(1)},\hat\beta^{(2)},...,\hat\beta^{(T)}]$. 
We denote the left top-r singular vectors of $\hat B$ as $\hat W_1$. Let $\mathbb{O}^{r\times r}$ be the class of all $r\times r$ orthonormal matrices. We will show that $\min_{O\in\mathcal{O}_{r\times r}}\| 
OW_1-\hat W_1\|_F$ vanishes as $n\to\infty$. 

To facilitate the theoretical derivation, we need a diversity assumption. Under the covariate shift setting, the distributions of $X_i^{(t)}$ are different over $t\in[T]$. For $g(u)=W_k(\sigma(W_{k-1}(\sigma(...\sigma(u+b_1))+b_{k-2})+b_{k-1})$ (for $u\in\R^{r}$), we define $m_t=\E_{{X^{(t)}}}[\nabla {g}(W_1X^{(t)})]$ and $M=[m_1,...,m_T]$.
The diversity assumption we impose here assumes that the $T$ marginal distributions of $X_i^{(t)}$ are diverse enough, that is, the $r$-th largest singular value $\lambda_{r}(MM^\top /T)>c$ for some universal constant $c>0$.  We remark here that similar diversity assumptions have been appeared in other transfer learning papers, e.g.  \cite{du2020few,tripuraneni2021provable}, where they consider two-layer neural network models.

  Under the diversity assumption mentioned above, we then present the following theorem. 
\begin{Theorem}\label{thm:sigmoid:rep:transfer}
Suppose for the $t$-th task, $X_i^{(t)}$'s are i.i.d.\,drawn from a Gaussian distribution $N_{d_1}(0,\Sigma_{X^{(t)}})$. We also assume $c_1\le\lambda_{\min}(\Sigma_{X^{(t)}})\le\lambda_{\max}(\Sigma_{X^{(t)}})\le c_2$ for some universal constants $c_1,c_2>0$. Moreover, for the function $g(\cdot)$ and matrix $M$ defined in the previous paragraph, we assume $g$ is $L$-Lipschitz with $L<C$ for some constant $C>0$, and $c_1'\le\lambda_{r}(MM^\top /T)>\lambda_{\max}(MM^\top /T)\le c_2'$ for some universal constants $c_1',c_2'>0$. Suppose $T=O(d_1^2)$. Letting $\hat h(x)=\hat W_1 x$, $h(x)=O W_1 x$ where $O=\arg\min_{O\in\mathcal{O}_{r\times r}}\| 
OW_1-\hat W_1\|_F$, we then have with probability at least $1-o(1)$, $$
\ee_{X}\|\hat h(X)-h(X)\|^2=O({\frac{r(d_1+T)}{nT}}).
$$
\end{Theorem}
\begin{proof}
We first use Stein's identity (Lemma~\ref{lm:stein}) and obtain $$
\bE_{{X^{(t)}}}[p^*(X^{(t)})X^{(t)}]=\Sigma_{X^{(t)}} W_1^\top \E[\nabla g(W_1X^{(t)})].
$$
We denote $\beta^{(t)}=\Sigma_{X^{(t)}}^{-1}\E[p^*(X^{(t)})X^{(t)}]= W_1^\top \E[\nabla g(W_1X^{(t)})]$ and $ B=[\beta^{(1)},\beta^{(2)},...,\beta^{(T)}]$. Then the diversity assumption implies that $\lambda_r(B)>c\sqrt{T}$. 

Moreover, using the same concentration analysis in the proof of Theorem~\ref{thm:ReLU:rep}, we have with probability at least $1-d_1^{-3}$, $$
\|\beta^{(t)}-\hat\beta^{(t)}\|\le \sqrt{\frac{d_1}{n}}.
$$
It follows that  with probability at least $1-Td_1^{-3}$, $$
\|B-\hat B\|_F\le \sqrt{\frac{d_1T}{n}}.
$$
In the following, we provide an upper bound on $\|B-\hat B\|_{2}$. 

Recall that $\hat\beta^{(t)}=(\frac{1}{n}\sum_{i=1}^n X^{(t)}_iX^{(t)\top}_i)^{-1}(\frac{1}{n}\sum_{i=1}^n X^{(t)}_iY^{(t)}_i)$.  We decompose $\hat\beta^{(t)}-\beta^{(t)}$ by
\begin{align*}
\hat\beta^{(t)}-\beta^{(t)}=&(\frac{1}{n}\sum_{i=1}^n X^{(t)}_iX^{(t)\top}_i)^{-1}(\frac{1}{n}\sum_{i=1}^n X^{(t)}_iY^{(t)}_i)-(\frac{1}{n}\sum_{i=1}^n X^{(t)}_iX^{(t)\top}_i)^{-1}\E[X^{(t)}Y^{(t)}]\\
+&(\frac{1}{n}\sum_{i=1}^n X^{(t)}_iX^{(t)\top}_i)^{-1}\E[X^{(t)}Y^{(t)}]-\Sigma_{X^{(t)}}^{-1}\E[X^{(t)}Y^{(t)}].
\end{align*}
Since $
\|B-\hat B\|=\sup_{\|u\|,\|v\|=1}u^\top(B-\hat B)v,
$
we then proceed to derive the concentration inequality for $u^\top(B-\hat B)v$ for any given $u\in\bR^d$, $v\in\bR^T$ satisfying $\|u\|,\|v\|=1$,
\begin{align*}
     u^\top(B-\hat B)v=&\sum_{t=1}^T v_t u^\top((\frac{1}{n}\sum_{i=1}^n X^{(t)}_iX^{(t)\top}_i)^{-1}(\frac{1}{n}\sum_{i=1}^n X^{(t)}_iY^{(t)}_i)-(\frac{1}{n}\sum_{i=1}^n X^{(t)}_iX^{(t)\top}_i)^{-1}\E[X^{(t)}Y^{(t)}])\\
+&\sum_{t=1}^Tv_t u^\top((\frac{1}{n}\sum_{i=1}^n X^{(t)}_iX^{(t)\top}_i)^{-1}\E[X^{(t)}Y^{(t)}]-\Sigma_{X^{(t)}}^{-1}\E[X^{(t)}Y^{(t)}]).
\end{align*}

Since we have with probability at least $1-Td_1^{-3}$,  $\lambda_{\min}(\frac{1}{n}\sum_{i=1}^n X^{(t)}_iX^{(t)\top}_i)\ge c-\sqrt\frac{d}{n}$ holds for all $t\in[T]$, implying that $$
\|(\frac{1}{n}\sum_{i=1}^n X^{(t)}_iX^{(t)\top}_i)^{-1}u\|\le C, \text{ for all } t\in[T].
$$
Then, using Berstein's inequality, we have with probability at least $1-\delta$, $$
 |\sum_{t=1}^T v_t u^\top((\frac{1}{n}\sum_{i=1}^n X^{(t)}_iX^{(t)\top}_i)^{-1}(\frac{1}{n}\sum_{i=1}^n X^{(t)}_iY^{(t)}_i)-(\frac{1}{n}\sum_{i=1}^n X^{(t)}_iX^{(t)\top}_i)^{-1}\E[X^{(t)}Y^{(t)}])|\le C' \sqrt\frac{\log(1/\delta)}{n}.
$$
Similarly, for the second term, using the fact that $$
(\frac{1}{n}\sum_{i=1}^n X^{(t)}_iX^{(t)\top}_i)^{-1}-\Sigma_{X^{(t)}}^{-1}=(\frac{1}{n}\sum_{i=1}^n X^{(t)}_iX^{(t)\top}_i)^{-1}((\frac{1}{n}\sum_{i=1}^n X^{(t)}_iX^{(t)\top}_i)-\Sigma_{X^{(t)}})\Sigma_{X^{(t)}}^{-1},
$$
and using Bernstein's inequality again, we obtain with probability at least $1-\delta$, $$
 |\sum_{t=1}^Tv_t u^\top((\frac{1}{n}\sum_{i=1}^n X^{(t)}_iX^{(t)\top}_i)^{-1}\E[X^{(t)}Y^{(t)}]-\Sigma_{X^{(t)}}^{-1}\E[X^{(t)}Y^{(t)}])|\le C \sqrt\frac{\log(1/\delta)}{n}.
$$

Combining these two pieces, we have that for any given $u\in\bR^d$, $v\in\bR^T$ satisfying $\|u\|,\|v\|=1$, we have $$
\Prob(u^T(\hat B-B)v)>C\sqrt\frac{\log(1/\delta)}{n})\le 2\delta.
$$

Then we use the $\epsilon$-net argument \cite{vershynin2018high}, and we obtain that with probability at least $1-T^{-1}-d_1^{-1}-Td_1^{-3}$, 
$$
\|B-\hat B\|_{}\le C \sqrt{\frac{d_1+T}{n}}.
$$
We then invoke the following lemma.
\begin{Lemma}[A variant of Davis–Kahan Theorem \cite{yu2015useful}]\label{lm:dk}
Assume $\min\{T,d\}>r$. For simplicity, we denote $\hat{\sigma}_1\ge \hat{\sigma}_2\ge \cdots \ge \hat{\sigma}_r$ as the top largest $r$ singular values of $\hat B$ and  $\sigma_1\ge \sigma_2\ge \cdots \ge \sigma_r$ as the top largest $r$ singular values of $B$. Let $W_1^\top=(v_1,\cdots,v_r)$ be the orthonormal matrix consists of left singular vectors corresponding to $\{\sigma_i\}_{i=1}^r$ and $\hat W_1^{\top}=(\hat v_1,\cdots,\hat v_r)$ be the orthonormal matrix consists of left singular vectors corresponding to $\{\hat \sigma_i\}_{i=1}^r$. Then,  there exists an orthogonal matrix ${O}\in \bR^{r\times r}$, such that
$$\|{O}\hat{W_1}-W_1\|_F\lesssim \frac{(2\sigma_1+\|\hat{B}-B\|_{})\min \{\sqrt{r}\|\hat{B}-B\|_{},\|\hat{B}-B\|_{F}\}}{\sigma^2_{r}}.$$

\end{Lemma}
Plugging in the previously derived upper bound on $\|\hat B-B\|$ and $\|\hat B-B\|_F$, we can then have with probability $1-o(1)$,
\begin{align*}\|{O}\hat{W_1}-W_1\|_F\lesssim& \frac{(2\sigma_1+\|\hat{B}-B\|_{})\min \{\sqrt{r}\|\hat{B}-B\|_{},\|\hat{B}-B\|_{F}\}}{\sigma^2_{r}}\\
\le& C \sqrt{\frac{r(d_1+T)}{nT}}.
\end{align*}

\end{proof}

\paragraph{Concept shift.}
Now we consider the concept shift setting. In this setting, where the conditional distribution on $Y$ given $X$ varies across source tasks, we assume that in the $t$-th source task, $t \in [T]$, we observe $n$ $i.i.d.$ samples from the model 
\begin{equation}\label{eq:model:transfer2}
p_t^*(x)=\Pr[Y^{(t)}=1\mid X^{(t)}=x]=W_k^{(t)}(\sigma(W^{(t)}_{k-1}(\sigma(...\sigma(W_1 x+b^{(t)}_1))+b^{(t)}_{k-2})+b^{(t)}_{k-1}):=g_t(W_1 x),
\end{equation}

where $x\in\R^{d_1}$, $W_1\in\R^{r\times d_1}$, $W^{(t)}_2\in\R^{d_2\times r}, W^{(t)}_j\in\R^{d_j\times d_{j-1}}$ for $j=3,...,k-1$, $W^{(t)}_k\in\R^{1\times d_{k-1}}$, $b^{(t)}_1\in\R^r, b^{(t)}_k\in\R^{d_k}$, and $\sigma(\cdot)$ is a general activation function.
We can then use the exactly same method as described for the covariate shift setting to learn $W_1$.

In this setting, the theoretical derivation would require a different but similar diversity assumption. The new diversity assumption we need here assumes that the $T$ conditional distributions $\Pro(Y_i^{(t)}\mid X_i^{(t)})$ are diverse enough, or formally speaking,  
the $r$-th largest singular value $\lambda_{r}(\tilde M\tilde M^\top/T)>c$ for some universal constant $c>0$, where $\tilde M=[\tilde m_1,...,\tilde m_T]$, $\tilde m_t=\E_{X}[\nabla g_t(W_1X)]$. 
{Here, as defined in \eqref{eq:model:transfer2}, $g_t(u) =W_k^{(t)}(\sigma(W^{(t)}_{k-1}(\sigma(...\sigma(u+b^{(t)}_1))+b^{(t)}_{k-2})+b^{(t)}_{k-1})$, for $u\in\R^{r}$.}

\begin{Theorem}\label{thm:sigmoid:rep:transfer2}
Suppose for the $t$-th task, $X_i^{(t)}$'s are i.i.d.\,drawn from a Gaussian distribution $N_{d_1}(0,\Sigma_{X})$. We also assume $c_1\le\lambda_{\min}(\Sigma_{X})\le\lambda_{\max}(\Sigma_{X})\le c_2$ for some universal constants $c_1,c_2>0$. Moreover, for the function $g_t(\cdot)$ and $\tilde M$ defined in the previous paragraph, we assume $g_t$'s are $L$-Lipschitz with $L<C$ for some constant $C>0$, and $c_1'\le\lambda_{r}(\tilde M\tilde M^\top /T)>\lambda_{\max}(\tilde M\tilde M^\top /T)\le c_2'$ for some universal constants $c_1',c_2'>0$. Suppose $T=O(d_1^2)$. Letting  $h(x)=O W_1 x$ where $O=\arg\min_{O\in\mathcal{O}_{r\times r}}\| 
OW_1-\hat W_1\|_F$, we then have with probability at least $1-o(1)$, $$
\ee_{X}\|\hat h(X)-h(X)\|^2=O({\frac{r(d_1+T)}{nT}}).
$$
\end{Theorem}
\begin{proof}
The proof is quite similar to that of Theorem~\ref{thm:sigmoid:rep:transfer}. We again begin by using Stein's identity (Lemma~\ref{lm:stein}) to obtain $$
\E[p^*(X^{(t)})X^{(t)}]=\Sigma_{X} W_1^\top \E[\nabla g_t(W_1X^{(t)})].
$$
We denote $\beta^{(t)}=\Sigma_{X}^{-1}\E[p^*(X^{(t)})X^{(t)}]= W_1^\top \E[\nabla g_t(W_1X^{(t)})]$ and $ B=[\beta^{(1)},\beta^{(2)},...,\beta^{(T)}]$. Then the diversity assumption implies that $\lambda_r(B)>c\sqrt{T}$. We can also see that $W_1^\top$ is the left top-$r$ singular vectors of $B$. The rest of the proof is essentially the same to the proof of Theorem~\ref{thm:sigmoid:rep:transfer}. 
\end{proof}

\section{The Power of Our Approach}
\label{sec: benefit}
In this Section we give two nontrivial examples of a data generation model for which, by using a neural net of a given class $\mathcal C$, together with training data, we can solve the \scaffolding{} set problem for $p^*$ even though $p^*$ itself cannot be computed by any neural net in~$\mathcal C$.

In the first example, we show that for any $k$-layer homogeneous neural network with bounded parameters ($k\in \bN$), there exists $p^*$ that cannot be approximated well by any $k$-layer neural network in that family, but there exists a $k$-layer neural network in that family, for which, if we apply the \Scaffolding{} Set algorithm using the mapping defined by any depth~$j$ prefix, $j=1,\cdots, k-1$, together with sufficient training data,
we can recover $p^*$ well. In addition, when $j=1$,  the representation mapping $\hat{h}$ can be found using the method of Section~\ref{sec: learning h}. 

{The second example is based on a result of Eldan and Shamir~\cite{eldan2016power} separating the computational power of depth~2 neural networks of polynomial (in the data dimension $d$) width and depth~3 networks of polynomial width. In particular, we adapt their construction to obtain a specific $p^*$ that can be expressed by a three-layer neural network of width polynomial in~$d$, but which cannot be approximated by any two-layer neural network of sub-exponential width.} 

{As a result, if $\cC$ is a family of two-layer neural networks of moderate width, we cannot train any neural network in $\cC$ to approximate $p^*$.  In contrast, using the Meta-Algorithm with input $\hat h$ being the first layer of a specific two-layer neural network of only polynomial width,  we obtain a predictor that approximates $p^*$ very well.
} 

{These examples demonstrate that, while neural networks in $\cC$ are insufficiently powerful and cannot approximate a specific $p^*$ well, the \Scaffolding{} Set algorithm can nonetheless leverage the partial structure recovered by a specific neural network in $\cC$ to construct sets leading, via multi-calibration, to a good estimator of $p^*$.}

\subsection{Example of $k$-layer homogeneous neural networks}

Fix an input dimension $d$ and consider the class $\mathcal C$ of neural nets of input dimension $d$ and depth $k$ of the following form: $f_k(x)=\sum_{i=1}^l w_{i}^\top \sigma(\phi_{k-1}(x))$, where $\sigma$ is the ReLU function, $w_i \in \bR^l$, $\phi_{k-1}(x)=A_{k-1}\sigma(A_{k-2}\sigma(\cdots\sigma(A_1x)))$ is a homogeneous neural network, and $A_i\in\bR^{l_i\times l_{i-1}}$ for $i=2,\cdots, k-2$ and $A_1\in\bR^{l_1\times d}$, $A_{k-1}\in\bR^{l\times l_{k-2}}$. In addition, all the parameters are bounded by a universal constant $C>0$, for instance, all the entries of the matrices $A_i$, $k=1,\dots,k$, belong to $[-C,C]$.
\begin{Remark}
Note that $f^*$ is a $k$-layer neural network, where its first $k-1$ layers belong to the same class as the first $k-1$ layers of $f_{k}$. However, the last layer of $f^*$ does not belong to the class of the last layer of $f_k$. 
\end{Remark}

{The next theorem shows the existence of a $p^*$ that cannot be approximated well by any neural net in $\cC$; however, the output of the $(k-1)$-th layer of a specific neural network in $\cC$ provides a low-dimensional representation mapping $\hat h$ that, when given as input to the Meta-Algorithm, recovers $p^*$ very well.}

\begin{Theorem}\label{thm:k layer}
There exists a universal constant $c>0$, a distribution $\cD$ on $\rr^d$ that is sub-Gaussian and has positive density on a compact support, and a probability function $p^*: \rr^d \rightarrow [0,1]$ of the form 
$$p^*(x) = f^*(x),$$ 
where $f^*(x)=w^{*\top}\sigma (\phi^*_{k-1}(x)),$ and $\phi_{k-1}^*(x)=A^*_{k-1}\sigma(A^*_{k-2}\sigma(\cdots\sigma(A^*_1x)))$ is a realization of the architecture of $k-1$ layer of the homogeneous ReLU neural network $\phi_{k-1}(x)$  with parameters bounded in $[-C,C]$ and $w^*$ is a vector with $\ell_2$-norm equals to $2Cl^{1.5}$, such that for any $f_k(x)$,
$$\ee_{X\sim \cD}|p^*(X)-  f_k(X)|^2\ge c .$$

Moreover, assume $k< C_1$ for a universal constant $C_1>0$, and suppose we apply the \Scaffolding{} Set Algorithm to input $\hat h$ satisfying Assumption~\ref{assumption on h} with respect to $h(x)=A^*_{i}\sigma(A^*_{i-1}\sigma(\cdots\sigma(A^*_1x)))$, where $i=1,\cdots,k-1$ (when $i=1$, $h(x)=\sigma(A^*_1x)$) to
obtain a collection of sets $\{S^{(\hat{h})}_k\}$
Then for any $\hat p$ that is $\alpha_n$-multi-calibrated with respect to $\{S^{(\hat{h})}_k\}$, and $\forall \delta>0$, if the sample size $n=\Omega(poly( l, 1/\delta))$, we have 
$$\ee_{X\sim \cD}|p^*(X)- \hat{p}(X)|^2\le \delta^2+ \upsilon e_{h}^2+\alpha_n^2,$$
where $\upsilon =O(poly(l))$.
\end{Theorem}
\begin{proof}
Let $\bR_C$ denote the interval $\bR\cap [-C,C]$. Let $\cF_{k-1}$  denote the class of functions computable by neural nets of the form
$$\cF_{k-1} = \{\phi(x)|\phi(x)=A_{k-1}\sigma(A_{k-2}\sigma(\cdots\sigma(A_1x))), A_i\in\bR_C^{l\times l}~\text{ for}~ i=2,\cdots, k-1~ \text{and\;} A_1\in\bR_C^{l\times d}\}.$$ 	
	
For all $x\in\rr^d$, we define
$$\phi^{(x)}_{k-1}= \argmax_{\phi\in \cF_{k-1}} \|\sigma(\phi(x))\|.$$

Since all the parameters are bounded, by continuity of $\phi$, the maximum can be achieved. If there are multiple maximizers, we can arbitrarily 
choose one.

Since $\cF_{k-1}$ is a collection of functions computable by homogeneous neural networks, we know that for any given $\lambda>0$ and all $x \in \rr^d$,
\begin{equation}\label{eq: linearity}
\argmax_{\phi\in \cF_{k-1}} \|\sigma(\phi(x))\|=\argmax_{\phi\in \cF_{k-1}} \|\sigma(\phi(\lambda x))\|
\end{equation}
Thus\footnote{Equation~\ref{eq: linearity} is where we use the fact that these are ReLU networks.}, for any constant $a > 0$, we can choose $x^*$ such that $\|\sigma(\phi^{(x^*)}_{k-1}(x^*))\|=a$.  Specifically, we choose $x^*$, such that $\|\sigma(\phi^{(x^*)}_{k-1}(x^*))\|=c_0l^{-1.5}$, for a constant $c_0$ that we will later specify and denote $\sigma(\phi^{(x^*)}_{k-1}(\cdot))$ simply by $\sigma(\phi^*_{k-1}(\cdot))$.

Next, notice that for any $f_k$ and $x$,  we have
$$\|f_k(x)\|\le Cl^{3/2}\|\sigma(\phi^{(x)}_{k-1}(x))\|.$$

Let us choose $\lambda = 2Cl^{3/2}$ and set $w^*= \lambda\sigma( \phi^*_{k-1}(x^*))/\|\sigma(\phi^*_{k-1}(x^*))\|$.  Now, we can choose the distribution $\cD$ on $\rr^d$ so that $\phi^*_{k-1}(\cdot)$ is uniformly distributed among $\{y:\|y-\sigma( \phi^*_{k-1}(x^*))\|\le \epsilon \|\sigma(\phi^*_{k-1}(x^*))\|\}$ for some small enough real number $\epsilon\le 0.05$. By the continuity of the functions in $\cF_{k-1}$ and the activation function $\sigma$, we know such distribution must exist and can make it has bounded support, thus sub-Gaussian.
Recall that we chose $x^*$ such that $\|\sigma(\phi^{(x^*)}_{k-1}(x^*))\|=c_0l^{-1.5}$.  We therefore have \begin{align*}
	\bE_{x\sim \mathcal D}(p^*(x)-p_k(x))^2&{\ge}\left |{[2Cl^{3/2}\|\phi^*_{k-1}(x^*))\|-2C\epsilon l^{3/2}\|\phi^*_{k-1}(x^*))\|]}
	-{{Cl^{3/2}\|\phi^*_{k-1}(x^*))\|}}\right|^2 \\
	&\ge\left |{{1.9Cl^{3/2}\|\phi^*_{k-1}(x^*))\|}}
	-{{Cl^{3/2}\|\phi^*_{k-1}(x^*))\|}}\right|^2  \\
	&= |{{0.9Cl^{3/2}\|\phi^*_{k-1}(x^*))\|}}|^2\\
	&\ge0.81 c_0^2C^2,
\end{align*}

Choosing $c_0<1/(3C)$, we then have over the distribution $\mathcal D$ defined the same way as before, $p^*(x)\in(0,1)$, and $\bE_x(p^*(x)-p_k(x))^2\ge c$.

Notice the Lipchitz constant of the function is bounded by $2C l^{3/2}$. The rest of the theorem follows directly by applying Theorem \ref{thm:mc wrt S}.
\end{proof}

According to the above theorem, if $e_{h}$ and $\alpha_n$ vanish fast enough when $n$ goes to infinity, with moderately large sample size, the {\algfull}(Algorithm~\ref{alg:full}) 
can be used to obtain an estimator $\hat{p}$, based on the output of some strict prefix of the network {\it i.e.}, the output of layer $k-1$ or some earlier layer, that is close to $p^*$, even though, in contrast, no $k$-layer neural network can approximate $p^*$ very well. A direct implication is the following corollary. We present informally for brevity.

\begin{Corollary}[Informal]
If we only consider networks of constant depth then we can obtain a suitable single-layer $\hat h$ via Theorem~\ref{thm:ReLU:rep} such that multi-calibration with respect to the output of the \Scaffolding{} Set algorithm, applied to this $\hat h$ together with a training set of moderate size, yields a good approximation to~$p^*$.  This holds despite the fact that no depth~$k$ network can approximate~$p^*$.  
\end{Corollary}
\subsection{Example of inhomogeneous shallow neural networks}

Let us briefly introduce the structure of neural networks we consider in this Section. Let $d$ be the input dimension.
\begin{itemize}

\item Two-layer neural network of width $l$: $x\mapsto \sum_{i=1}^l v_i\sigma(w_i^\top x +b_i)$, where $v_i\in \bR$ and $w_i \in \bR^d$.

\item Three-layer neural network of width $l$: $x\mapsto  \sum_{i=1}^l u_i\sigma \left(\sum_{j=1}^l v_{i,j}\sigma(w_{i,j}^\top x +b_{i,j})+c_i\right )$, where $u_i,c_i,b_{i,j}\in \bR$, $w_{i,j}\in \bR^d$.
\end{itemize}
The results in this Section apply to the case in which $\sigma$ is either the ReLU function $\sigma(x) = \max\{0,x\}$ or the Sigmoid function $\sigma(x) = \log(1+e^x)$ .

The following theorem builds on a result of Eldan and Shamir, that separates the power of polynomial-width ReLU neural nets of depth~2 and depth~3.  Specifically, they demonstrated the existence of a function $f$ and a distribution $\mathcal D$ such that $f$ is computable by an inhomogeneous ReLU network of polynomial width and depth~3, but cannot be approximated, in expectation over $\mathcal D$, by any polynomial-width ReLU network of depth~2.

\begin{Theorem}\label{thm:2vs3}
There exists a distribution $\cD$ on $\rr^d$ and constants $C,c, c'$, such that the following properties hold. Suppose the input dimension $d>C$, there exists a universal constant $\tilde{c}>0$ such that for any $\delta\in(0,1)$, there exists a probability function $p^*(x)=   \sum_{i=1}^{l^*} u^*_i\sigma \Big(\sum_{i=1}^{l^*} v^*_{i,j}\sigma(w_{i,j}^{*\top} x +b^*_{i,j})+c^*_i\Big)$, with $l^* \le 64 c'/\tilde{c} d^5+1$ so that the following holds.
	\begin{itemize}
\item[1.] If $l$ is an integer satisfying $l\le ce^{cd}$, then for any two-layer neural network $f$, \text{i.e.} $f(x)=\sum_{i=1}^l v_i\sigma(w_i^\top x +b_i),$ we have that
$$\bE_{x\sim \cD}|f(x)-p^*(x)|^2\ge \tilde{c}^2/4.$$

\item[2.] However, if we use our method to build sets upon $\hat{h}$ satisfying learnablility assumption 4 (will modify later for assumption part) for  $h=\sum_{j=1}^l v^*_{i,j}\sigma(w_{i,j}^{*\top} x +b^*_{i,j}),$ we can have that for any $\hat{p}$ that is $\alpha_n$-multi-calibrated on $\{S^{(\hat{h})}_k\}$, and if sample size $n=\Omega(poly( d,l^*,1/\delta))$, we have that
$$E_{x\sim \cD}|p^*(x)- \hat{p}(x)|^2\le \delta^2+ \upsilon e_{h}^{2}+\alpha_n^2,$$
where $\upsilon =O(poly(d,l^*,1/\tilde{c}))$.
	\end{itemize}
\end{Theorem}

\begin{proof}
In Proposition $13$ and $17$ in~\cite{eldan2016power}, Eldan et.al.\,state when the activation function is ReLU or Sigmoid activation function, there exists a function $g(x)$ computable by a narrow three-layer neural network that cannot be approximated by a 2-layer neural net with width that is polynomial regarding input dimension.  More formally, there exists a distribution $\cD$ and constants
$C,c, c'>0$, 
such that the following properties hold. Suppose the input dimension $d>C$, there exists a constant $\tilde{c}>0$ such that for any $\eta\in(0,1)$, there exists a three-layer neural network $g$, i.e.
$$g=\sum_{i=1}^{l_g} u^g_i\sigma \left(\sum_{j=1}^{l_g} v^g_{i,j}\sigma(w_{i,j}^{g\top} x +b^g_{i,j})+c^g_i\right )$$

with width $l_g \le 8c'/\eta d^5+1$ so that the following holds. 

\begin{itemize}
    \item $g(x)\in [-2,2]$.
    \item    For all integers $l$ satisfying $l\le ce^{cd}$ and  any two-layer neural network $f$, \text{i.e.}
$$f(x)=\sum_{i=1}^l v_i\sigma(w_i^\top x +b_i),$$ 	
we have
$$(\bE_{x\sim \cD}|f(x)-g(x)|^2)^{1/2}\ge \tilde{c}-\eta.$$
\item $\sum_{i=1}^{l_g} u^g_i\sigma \left(x+c^g_i\right )$ is $\gamma$-Lipschitz as a function of $x$, where $\gamma=O(poly(d,l^*,1/\tilde{c}))$.
\end{itemize}

Note that $g$ is already  bounded in $[-2,2]$, thus we only need to prove that $\frac{1}{4}g+\frac{1}{2}$ shares the same properties as $g$, then we can take 
$$p^*=\frac{1}{4}g+\frac{1}{2},$$
since $\frac{1}{4}\tilde{g}+\frac{1}{2}\in [0,1]$.

Consider $l$ to be an integer satisfying $l\le ce^{cd}$. We focus on studying any two-layer neural network $f$ with width $l-1$, \text{i.e.}
$$f(x)=\sum_{i=1}^{l-1} v_i\sigma(w_i^\top x +b_i).$$ 
Let us denote $L_2(\cD)$ as the $L_2$ norm regarding the distribution $\cD$. By directly plugging in, we have 

\begin{align*}
 \|f-p^*\|_{L_2(\cD)} = \frac{1}{4}\|4f-2-g\|_{L_2(\cD)}.
\end{align*}

Notice $4f-2$ can be expressed by a two-layer neural network with width $l$. For instance, for ReLU activation function, we can scale $v_i$'s for  $i=1,\cdots l$, by $4$ and add one single node for the extra $2$ by taking $v_l=-1, w_l=0, b_l=2$. For Sigmoid function, the transformation is similar.

Thus by the inapproximability of $l$-width two-layer neural network, we can obtain 
$$ \|f-p^*\|_{L_2(\cD)} = \frac{1}{4}\|4f-2-g\|_{L_2(\cD)}\ge \frac{1}{4}(\tilde{c}-\eta).$$

If we take $d$ to be large enough such that $ce^{cd}-1>(c/2)e^{cd/2}$, then we know for any $l'\le (c/2)e^{cd/2}$, we must have $l'+1\le ce^{cd}$ and any two layer neural network of width $l'$

$$f(x)=\sum_{i=1}^{l'} v_i\sigma(w_i^\top x +b_i),$$ 
then by the inapproximability of $l'+1$-width two-layer neural network, we have
\begin{align*}
 \|f-p^*\|_{L_2(\cD)} \ge \frac{1}{4}(\tilde{c}-\eta).
\end{align*}

Let us take $\eta=\tilde{c}/2$, the proof of bullet point $1$ is complete.

The proof of bullet point $2$ is a straightfoward application of Theorem \ref{thm:mc wrt S}.

\end{proof}
According to the above theorem, if $e_{h}$ and $\alpha_n$ vanish fast enough when $n$ goes to infinity, with moderately large sample size, using the {\algfull}
(Algorithm~\ref{alg:full}) 
we can obtain an estimator $\hat{p}$ that is close to $p^*$ while using any two-layer neural network cannot approximate $p^*$ very well unless the width is very large.

\begin{Corollary}\label{cor:2vs3}
Under the same settings as in Theorem \ref{thm:2vs3}. For any two-layer neural network $f$ that is not wide enough (unless is exponentially wide with respect to input dimension), \text{i.e.} $f(x)=\sum_{i=1}^l v_i\sigma(w_i^\top x +b_i),$	
there exists a constant $c''$, such that
$$\bE_{x\sim N(0,I_d)}|f(x)-p^*(x)|^2\ge c''.$$
\end{Corollary}
\begin{Remark}
The conditions of Corollary ~\ref{cor:2vs3} satisfies the conditions of Theorem~\ref{thm:sigmoid:rep}, where we obtain an $\hat h$ and the resulting $\hat p$ satisfies $$
\bE_{x\in N(0,I_d)}|f(x)-p^*(x)|^2\lesssim (\frac{\log n}{n^{1/3}})^2+\alpha_n^2.$$
When $n$ is large enough, using our method can obtain a better estimation of $p^*$ than directly using two neural networks, which are not wide enough. Following our proof, we can further extend the results of \cite{eldan2016power} to a wide class of density as long as that density $p$ is lower bounded on a large enough but bounded support that contains a certain range.
\end{Remark}


\section{Experimental Support} 
To provide further support for the methods presented in Sections~\ref{sec:main_results} and~\ref{sec: learning h}, we present two numerical experiments. We define a mapping  $p^*(x)$ from instances in the data universe to probabilities, and for each sample $x \in D$ in the dataset, we generate the outcome $y$ such that $y \sim \mathrm{Bern}(p^*(x))$. Dividing the resulting dataset into a training and test set, we train a fully connected neural network on the training pairs $\{(x, y)\}$. 

Given access to both the $p^*$ and labels $y$ for each $x$, we can check for proximity to truth by computing the mean squared error (MSE) between the predictions of the fully connected network and the underlying $p^*$ used to generate the Boolean labels.  This is unlike the traditional setting, in which one only has the Boolean outcomes $y$ without access to the underlying probability distributions from which the outcomes are drawn.

In more detail, we consider the MNIST dataset of 50,000 training images and 10,000 test images of handwritten digits labeled 0 through 9. We define $p^*(x) = \frac{\#}{10}$; for example, an image $x$ of the digit 1 will have $p^*(x)=10\%$ and an image $x$ of the digit 9 will have $p^*(x)=90\%$. Motivated by Theorems \ref{thm:ReLU:rep} and \ref{thm:sigmoid:rep} in Section~\ref{sec: learning h}, we define $\hat h(x)$ to be the first $k-2$ layers of a $k$-layer a fully connected neural network. The network is trained on the Boolean outcomes defined by draws from $p^*(x)$ as described above. We expect that such an $\hat h(x)$ will learn a useful representation of the underlying data and that generating \scaffolding{} sets with this $\hat h(x)$ via Algorithm~\ref{alg:set} and multi-calibrating with respect to them as in Algorithm~\ref{alg:full} offers a performance benefit over the network itself. 

In this setting, the goals of our two experiments are as follows:

\textit{Demonstrating Learnability:} We validate the Learnabilty Assumption (Remark~\ref{rem:learnability}) for the $\hat h$ described above.  By training a neural network on the Boolean outcomes first, freezing the first $k - 2$ layers of the network, and then further training the remaining $2$ layers using the true $p^*$ values, we obtain a small suffix network $g$ such that $p^*(\cdot) \approx g({\hat h}(\cdot))$.  This demonstrates that the learned $\hat h(x)$ representation is sufficiently informative to be used as an input into a small neural network that can closely approximate the true $p^*(x)$ values. This justifies using $\hat h(x)$ as the input to the \scaffolding{} set algorithm in the next experiment. 

\textit{Demonstrating the Benefit of our Approach:} We demonstrate a case in which the multi-calibrating with respect to the \scaffolding{} sets found by Algorithm~\ref{alg:set} offers a benefit over a (traditionally) trained neural network. Training the network on the Boolean outcomes alone, removing the last two layers of the network, and then setting $\hat h(x) = h(x)$ and applying Algorithm~\ref{alg:full}, we show that multi-calibrating with respect to the resulting collection of \scaffolding{} sets offers a material benefit over using the original neural network itself in terms of closeness to the true underlying $p^*(x)$.

\subsection{Demonstrating Learnability}
In addition to the motivation provided by Theorems~\ref{thm:sigmoid:rep} and \ref{thm:ReLU:rep} in Section~\ref{sec: learning h}, we provide experimental justification for Assumption~\ref{assumption on h}. Recall that Assumption~\ref{assumption on h} requires that there exist some $w_{new}$ which, when composed with $h$, gives an estimator $p$ that is close to $p^*$. Then, so long as we can find a $\hat h$ that is close to $h$, we can use Algorithm~\ref{alg:set} to construct the \scaffolding{} sets. In this experiment, we show that we can find a $p = w_{new}\circ h$ that is close to $p^*$ and motivate taking $\hat h = h$. 

Recall that in our experimental setup, for an image $x$ of the digit $i$ we define $p^*(x)$ to be $i/10$, and for an image $x$ of digit $i$ in the training data the label $y$ is drawn from $Bern(p^*(x)) = Bern(i/10)$.

First, we train a fully connected network with the Boolean labels. We partition this network into $w$ and $h$, where $w$ is the last two layers, including the output layer, and $h$ is the remainder. We then show that, by freezing the layers in $h$ and training a new $w_{new}$ on the $p^*$ values directly, the model achieves a low mean squared error with respect to the underlying $p^*$ values. That is, we find that $w_{new} \circ h(x)$ is close to $p^*$. Since $w_{new}$ is narrow and shallow and therefore cannot alone perform the task of getting $w_{new} \circ h(x)$ close to $p^*(x)$, we find that $h(x)$ describes a reasonable representation of the input. Since $h$ was only trained on Boolean data, we can simply take $\hat h = h$ and run the \scaffolding{} algorithm with confidence in Assumption~\ref{assumption on h}. 

Specifically, we train a 6 layer fully connected neural network with widths 784, 512, 512, 256, 256, and 2 in the layers. The first two layers use a ReLU activation function while the remainder use the Sigmoid function. The network is trained for 75 epochs using stochastic gradient descent with an initial learning rate of 0.01, which decays by a factor of 10 at epochs 40 and 70. After the network is trained, $w$ is frozen and $w_{new}$ is trained for 50 epochs with a learning rate of 0.09 using $p^*$ values, as described above. Each point in Figure~\ref{fig:expt1} is the result of averaging over 10 model instances. The plot demonstrates the convergence of $p_{new} = w_{new} \circ h$ to $p^*$ and, therefore, the validity of choosing $\hat h = h$ in Algorithm~\ref{alg:set}.

\begin{figure}[H]
\includegraphics[width=0.35\textwidth]{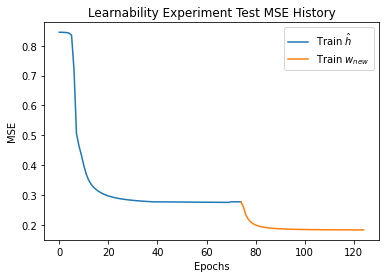}
\centering
\caption{\fontsize{10pt}{12}\selectfont
(Blue) Model performance on the test set as measured by the prediction MSE with $p^*(x)$. The model is trained on Boolean data only during this phase. (Orange) Model performance on the test set as measured by the prediction MSE with $p^*(x)$. The model is trained on the true probability labels in this phase, but only the final two layers of the network may be updated. The substantial drop in MSE during the second phase of training demonstrates the shortcomings of  neural networks alone in estimating $p^*$, as well as the utility of the $h(x)$ representation trained in the previous phase for extracting features relevant for modeling $p^*(x)$. Each point shown is averaged over 10 trials.}
\label{fig:expt1}
\end{figure}

\subsection{Demonstrating the Benefit of Our Approach}
After training the neural network (denoted by $\hat p_{NN}$), we extract the output of the second to last layer, with 4 neurons, providing us with a four-dimensional  representation $\hat h(x)$ for each training sample. For a given $B$, we then apply Algorithm~\ref{alg:set} to $\hat h(x)$ to partition the training data into $B^4$ sectors and compute $\hat p(x)  = \frac{1}{|S_{G(x)}|} \sum_{X_i \in S_{G(x)}} Y_i$ with partitioning scheme $G(x)$. Using this $\hat p(x)$ on the test data, we can compute the MSE with respect to the true $p^*$ values and compare it to the MSE of the direct output of the fully connected network ($\hat p_{NN}$) with respect to $p^*$ to determine whether Algorithm~\ref{alg:full} gets closer to the truth than using the neural network directly.

The results for $B \in \{1, ..., 9\}$ are shown in Figure \ref{fig:expt2}, demonstrating that for $B$ between 2 and 6, Algorithm~\ref{alg:full} outperforms the fully connected network in mean squared error with $p^*$. For larger values of $B$, the algorithm begins to overfit (note that $B = 9$ corresponds to $6561$ partitions, or an expected $7.6$ training samples per partition). Figure~\ref{fig:expt2}, combined with the results of Experiment 2, demonstrates that the fully connected network has indeed learned a useful representation of the underlying data, however, guaranteeing that the algorithm is multi-calibrated with respect to the \scaffolding{} sets obtained by Algorithm~\ref{alg:set} can offer an improvement over an standard neural network.

\begin{figure}[H]
\includegraphics[width=0.35\textwidth]{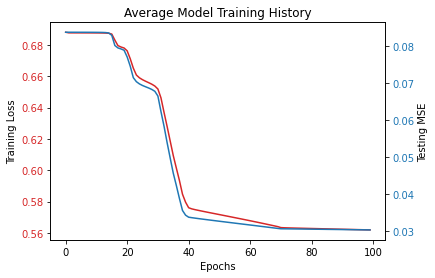}
\includegraphics[width=0.31\textwidth]{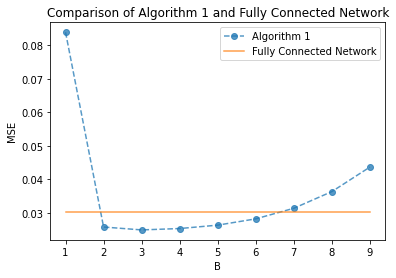}
\centering
\caption{\fontsize{10pt}{12}\selectfont (Left) Training history of the model in Experiment 2. The training cross entropy loss, between the Boolean outcomes and the model predictions, is shown in red while the test MSE, between the true $p^*(x)$ values and the model predictions, is shown in blue. (Right) Accuracy benefit of Algorithm~\ref{alg:full} over a standard neural network. For $B \in \{2, \dots, 7\}$, Algorithm~\ref{alg:full} outperforms the fully connected network. In both plots, each point shown is averaged over 10 trials.}
\label{fig:expt2}
\end{figure}

We further show that our results are robust to perturbations in the network architecture. Starting from a 5 layer ``base'' architecture with layers containing 784, 512, 512, 256, and 4 neurons respectively, we show that adding 1, 2, 3, or 20 extra layers with 256 neurons each to the network (in the final case doubling the network size), Algorithm~\ref{alg:full} offers similar test mean squared error. Each network is trained for 100 epochs using SGD with an initial learning rate of 0.1, decaying by a factor of $\frac{1}{\sqrt{10}}$ every 10 epochs between epochs 20 and 70 inclusive. The network with 20 additional layers is trained with Adam and an initial learning of $1\times10^{-4}$. The results are shown in Figure~\ref{fig:expt3} and demonstrates that Algorithm~\ref{alg:full} performs similarly for each of these choices of $\hat h$.

\begin{figure}[H]
\includegraphics[width=0.35\textwidth]{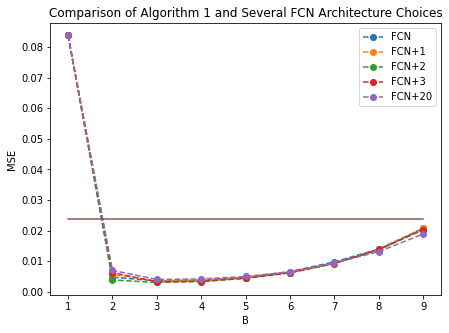}
\centering
\caption{\fontsize{10pt}{12}\selectfont The impact of changes to neural network architecture on the performance of Algorithm~\ref{alg:full}. Five fully connected networks (FCNs) of each architecture are trained to generate five $\hat h(x)$ functions, each of which is used as an input to Algorithm~\ref{alg:full}. The average test set performance of each is shown for varying choices of $B$ is shown in the figure above while the solid horizontal line denotes the average performance of the original neural networks themselves on the test set data. Here, we define the test set as the ground truth $p^*(x)$ values. The labels ``FCN+x'' refer to the architecture of the first FCN with the addition of $x$ additional dense layers.}   
\label{fig:expt3}
\end{figure}

\section{Conclusions and Future Directions}
This work presents a proof of concept that multi-calibration can be carried out efficiently {\em and} simultaneously approximate $p^*$, through the lens of representation learning, and suggests several directions for future research.  We briefly discuss two. 

In our work, we theoretically analyze the feasibility of using the first layer of a neural network as the representation. We expect that more intermediate layers can provide even better representations. Analyzing intermediate layers of neural networks is a general challenge in neural network theory; our results provide added motivation for this field of study. 
We have also provided two representation learning methods that fit a smaller (low-complexity) model to obtain a low-dimensional representation function. 
An interesting direction would be to consider other explicit dimensional reduction methods, for example, PCA and autoencoders.

\section{Acknowledgements}
The authors thank Michael Kim, Omer Reingold, Guy Rothblum, and Gal Yona for conversations that inspired this work.

\bibliography{References.bib, example_paper.bib}
\bibliographystyle{abbrv}

\newpage 
\appendix

\section{More Proofs}


\subsection{Proof of Lemma~\ref{lem:relu}}
Recall that $X$ is symmetric, so $W_1 X$ is symmetric. As a result, we have \begin{align*}
\E[YX]=&\E[p^*(X)X]=\E[q(\sigma(W_1X))X]\\
=&\frac{1}{2}\E[q(\sigma(W_1X))X\mid W_1 X\ge 0]+\frac{1}{2}\E[q(\sigma(W_1X))X\mid W_1X<0]\\
=&\frac{1}{2}\E[q(\sigma(W_1X))X\mid W_1 X\ge 0]\\
=&\frac{1}{2}\E[q(1)XX^\top W_1^\top\mid W_1 X\ge 0]\\
=&\frac{1}{2}\E[q(1)XX^\top W_1^\top\mid W_1 X< 0]\\
=&\frac{1}{2}\E[q(1)XX^\top W_1^\top]=\frac{1}{2} q(1)\cdot\Sigma_XW_1^\top.
\end{align*}

\subsection{Proof of Corollary \ref{cor:2vs3}}

\paragraph{High-level intuition.} Before presenting our formal proof, we illustrate the high-level idea of \cite{eldan2016power}, and our modification. First, they construct a radial function (only as a function of $\|x\|$) with bounded support, denoted $\tilde{g}$. Then, they present a method to construct a three-layer neural network to approximate $\tilde{g}$. Specifically, they first use linear combinations of neurons to approximate $z\mapsto z^2$. Then, adding these combinations together, one for each coordinate, one can compute $x\mapsto\|x\|^2 =\sum_{i=1} x_i^2$ inside any bounded domain. They last layer computes a univariate function of $\|x\|$. 

Now, for any density function $\varphi^2$, and any candidate function $f$ to approximate $\tilde g$, we evauate the approxiamtion via the $L_2$ distance under the density function $\varphi^2$:
$$\int (f-\tilde g)^2\varphi^2 dx =\int (f\varphi-\tilde g\varphi)^2 dx .$$
If the Fourier transformation of $f\varphi-\tilde g\varphi$ exists, then
$$\int (f\varphi-\tilde g\varphi)^2 dx=\int (\widehat{f\varphi}-\widehat{\tilde g\varphi})^2dx$$
(here, $\hat{\cdot}$ denotes the Fourier transformation).

The heart of the argument by Eldan et.al.is as follows.
If $f$ is a two-layer neural network, then the support of $\widehat{f\varphi}$ is a union of tubes of bounded radius passing through the origin, with the number of tubes exactly corresponding to its width (Lemma~\ref{lm:span}). The construction of $\widehat{\tilde g\varphi}$ is radial and has relative large mass in all directions. Thus, to approximate $\widehat{\tilde g\varphi}$ well will require many tubes, which means that $f$ must be very wide.

In our construction, we further show that for any not too wide two-layer neural network $f$ and indicator function $I$ for any ball with any radius, $\widehat{fI\varphi}$ also cannot approximate $\widehat{\tilde g \varphi}$ well. Since $\tilde g$ is of bounded support, if we choose the indicator function of ball with large enough radius, then $\tilde g I=g$. Thus, we can show that $\widehat{fI\varphi}$ cannot approximate $\widehat{\tilde g I \varphi}$ well. Thus, no
bounded range $f$ can approximate $\tilde g$ well. Finally, we obtain our result by applying Equation \ref{eq:is}.

\paragraph{Formal proof.}Now, let us formally prove our statements. We focus on proving the case where the activation function is the ReLU function or Sigmoid function.

The two-layer neural network is of the form:
$$f(x)=\sum_{i=1}^l v_i\sigma(w_i^\top x +b_i).$$
The form can also be abbreviated as
$$f(x)=\sum_{i=1}^lf_i(\langle x,u_i \rangle)$$
for some unit vector $u_i \in \bR^d$.

We further introduce some notation. Let $B_d$ be the $d$-dimensional ball with unit radius and center $0$. The density function used in \cite{eldan2016power} is $\varphi^2(x)$, where 
$$\varphi(x)=(\frac{R_d}{\|x\|})^{d/2}J_{d/2}(2\pi R_d\|x\|) ,$$
where for non-negative integer $\alpha$, $J_\alpha$ is the corresponding Bessel function and $R_d = \sqrt{1/\pi}(\Gamma(d/2+1))^{1/d}$. We further denote $$\beta_R(x)=(\frac{R}{\|x\|})^{d/2}J_{d/2}(2\pi R\|x\|).$$
As shown in Lemma $2$ in \cite{eldan2016power}, $\beta_R$ is the Fourier transformation of the indicator function $1\{x\in RB_d\}$ and by duality theory, we have $$\hat{\beta}_R(w)=1\{w\in RB_d\} ,$$
where we use $\hat \cdot$ to denote the Fourier transformation throughout the proof.

\textit{Generalized Fourier Transformation}: let $\cS$ denote the space of Schwartz functions (functions with super-polynomial decaying values and derivatives) on $\bR^d$. A tempered distribution $\mu$ in our context is a continuous linear operator from $\cS$ to $\bR$. In particular, any measurable function $h:\bR^d\mapsto \bR$, which satisfies a polynomial growth condition:
$$|h(x)|\le C_1(1+\|x\|^{C_2})$$
for universal constants $C_1,C_2>0$, $h$ can be viewed as a tempered distribution defined as 
$$\psi\mapsto \langle h,\psi\rangle:=\int_{\bR^d}h(x)\psi(x)dx,$$
where $\psi\in S$. The Fourier transformation $\hat h$ of a tempered distribution $h$ is also a tempered distribution, and defined as
$$\langle\hat h,\psi\rangle:=\langle h,\hat \psi\rangle,$$
where $\hat \psi$ is the Fourier transformation of $\psi$. In addition, we say a tempered distribution $h$ is supported on some subset of $\bR^d$, if $\langle h,\psi\rangle=0$ for any function $\psi\in\cS$ which vanishes on that subset.

Let us define $\tilde{f}_i(x)=f_i(\langle x,u_i\rangle)$. We define the following lemma.

\begin{Lemma}\label{lm:span}
Let $f=\sum_{i=1}^lf_i(\langle x,u_i \rangle)$ be a function such that $f1\{RB_d\}\varphi\in L_2$, and 
$$f_i(x)\le C(1+\|x\|^{\alpha})$$
for constants $C,\alpha>0$.

Then,
$$Supp(\reallywidehat{f1\{RB_d\}\varphi}) \subset \cup_{i=1}^l(Span\{u_i\}+R_dB_d).$$
\end{Lemma}
\begin{Remark}
The choice of ReLU and Sigmoid activation function satisfies the polynomial growth condition for $f_i$. Meanwhile, this lemma suggests that $Supp(\reallywidehat{f1\{RB_d\}\varphi})$ is contained in a union of ``tubes" passing through the origin. We can see that the number of tubes is exactly the width $l$, and this lemma suggests the limitation of the approximation power of two-layer neural networks.
\end{Remark}
\begin{proof}
First, we notice that $\beta_R\in L_2$ and $\psi\in L_1$ for any $\psi\in \cS$, then by convolution theorem, we have that 
$$\widehat{\beta_R\star\psi}=\hat{\beta}_R\hat{\psi},$$
where $\star$ is the convolution notation.

Secondly, we have that 
\begin{align*}
 \langle \tilde{f}_i1\{RB_d\}, \hat{\psi}\rangle &=\int \tilde{f}_i(x)1\{x\in RB_d\}\hat{\psi}(x)dx\\
 &=\int \tilde{f}_i(x)\hat{\beta}_R(x)\hat{\psi}(x)dx\\
 &= \int \tilde{f}_i(x)\widehat{\beta_R\star\psi}(x)dx.
\end{align*}

By the similar calculation in Claim $2$ in \cite{eldan2016power}, we have that 
\begin{equation}\label{eq:1}
\int_{\bR^d}\tilde{f}_i(x)\hat{\gamma}_R(x)dx = \int_{\bR} f_i(y)\hat{\gamma}_R(yu_i)dy
\end{equation}
where $\gamma_R =\beta_R\star\psi$.

Next, for every $\psi\in\cS$, by definition of generalized Fourier transformation
$$\langle\reallywidehat{\tilde{f}_i1\{RB_d\}\varphi},\psi\rangle =\langle \tilde{f}_i1\{RB_d\}, \reallywidehat{1\{R_dB_d\}\star\psi}\rangle.$$

As a result, for $\phi=1\{R_dB_d\}\star\psi$ and any $\psi\in \cS$ that vanishes on $Span\{u_i\}+R_dB_d$, then $\phi$ must vanish on $Span(u_i)$. Let $\phi_i(y)=\phi(yu_i)$
\begin{align*}
\langle \reallywidehat{\tilde{f}_i1\{RB_d\}\varphi},\psi\rangle&=\langle  \tilde{f}_i1\{RB_d\}, \varphi\hat{\psi}\rangle\\
&=\langle  \tilde{f}_i1\{RB_d\}, \hat \phi\rangle\\
&= \langle  \tilde{f}_i, \hat{\beta}_R\hat \phi\rangle\\
&=\int f_i(y)\widehat{\beta_R\star\phi}(yu_i)dy\quad\text{(by Equation \ref{eq:1})}\\
&= \int f_i(y)1\{|y|\le RVol(B_d)\}\hat\phi (yu_i)dy\\
&=\int f_i(y)1\{|y|\le RVol(B_d)\}\hat{\phi}_i(y) dy\\
&=0.
\end{align*}

The last equation is due to that by the definition of $\phi_i$, we know that if $\phi$ vanishes on $Span\{u_i\}$, thus
$$\phi_i(x)=\phi(xu_i)=0,$$
which results in that $\hat\phi_i$ is a zero function.

As a result, since $f(x)=\sum_{i=1}^lf_i(\langle x,u_i \rangle)$
$$Supp(\reallywidehat{f1\{RB_d\}\varphi}) \subset \cup_{i=1}^l(Span\{u_i\}+R_dB_d).$$
Since $f1\{RB_d\}\varphi\in L_2$, the tempered distribution coincides with the standard one, the result follows.

\end{proof}

Next, we recall the following lemma.

\begin{Lemma}[Lemma $9$ in \cite{eldan2016power}]
Let $q$ and $w$ be two functions of unit norm in $L_2$. Suppose that $q$ satisfies
$$Supp(q)\subset \cup_{i=1}^l(Span\{v_i\}+R_dB_d)$$
for $l\in \bN$. Moreover, suppose that $w$ is radial and that $\int_{2R_dB_d} w^2(x)dx\le 1-\delta$ for some $\delta\in[0,1]$. Then,
$$\langle q,w\rangle_{L_2}\le 1-\delta/2+l\exp(-cd)$$
where $c>0$ is a universal constant.
\end{Lemma}

Finally, we consider $\tilde{g}(x)$ defined in Proposition $1$ in \cite{eldan2016power}. And the function $\tilde{g}$ is of bounded support $\cB$, thus, as long as $\cB\subset RB_d$, we have that 
$$\|f\varphi1\{RB_d\}-\tilde{g}\varphi 1\{RB_d\}\|_{L_2}=\|f\varphi1\{RB_d\}-\tilde{g}\varphi\|_{L_2}$$

Then, let us denote $q=\frac{\reallywidehat{f1\{R_dB_d\}\varphi}}{\|f1\{R_dB_d\}\varphi\|_{L_2}}$ and $w=\frac{\widehat{\tilde{g}\varphi}}{\|\tilde g\varphi\|_{L_2}}$ and by similar proof of Proposition $1$, we have that

\begin{align*}
\|f\varphi1\{RB_d\}-\tilde{g}\varphi 1\{RB_d\}\|_{L_2}=\|f\varphi1\{RB_d\}-\tilde{g}\varphi\|_{L_2}&=\|\|f1\{R_dB_d\}\varphi\|_{L_2} q-\|\tilde g\varphi\|_{L_2} w\|_{L_2}\\
&\ge \frac{1}{2}\|q-2\|_{L_2}\|\tilde{g}\|\\
&\ge \frac{1}{2}\sqrt{2(1-\langle q,w\rangle_{L_2})}\|\tilde{g}\varphi\|_{L_2}\\
&\ge C_3/\alpha
\end{align*}
for a universal constant $C_3$, where the last equation is due to Lemma $6$ and $7$ in \cite{eldan2016power}.

Recall that by Lemma 10 in \cite{eldan2016power}, $p^*$ constructed in Theorem \ref{thm:2vs3} has a uniform approximation to $\tilde{g}$ in the sense that 
$$\sup_x|p^*(x)-\tilde g(x)|\le c.$$
for a small constant $c$ (can choose to be smaller than $C_3/(100\alpha))$
Thus, for Gaussian distribution $N(0,I_d)$, we can choose $R$ such that $\cB\subset RB_d$, such that
$$\|f\varphi1\{RB_d\}-p^*\varphi 1\{RB_d\}\|_{L_2}\ge \|f\varphi1\{RB_d\}-\tilde{g}\varphi 1\{RB_d\}\|_{L_2}- \|\tilde{g}\varphi1\{RB_d\}-p^*\varphi 1\{RB_d\}\|_{L_2}\ge \frac{99C_3}{100\alpha}$$

Let $\tilde{p}$ denote the density function of Gaussian $N(0,I_d)$. We know that $\sqrt{\tilde{p}}/\varphi\ge C_5$ for a constant $C_5>0$ on $RB_d$
\begin{equation}
\label{eq:is}
\|f\sqrt{\tilde{p}} 1\{RB_d\}-p^*\sqrt{\tilde{p}} 1\{RB_d\}\|_{L_2}\ge C_5\|f\varphi1\{RB_d\}-p^*\varphi 1\{RB_d\}\|_{L_2}\ge  \frac{99C_3C_5}{100\alpha}.
\end{equation}
The proof is complete.

\end{document}